\documentclass[nohyperref]{article}

\usepackage{microtype}
\usepackage{graphicx}
\usepackage{booktabs} % for professional tables

\usepackage[accepted]{icml2023}

\usepackage{hyperref}

\hypersetup{
	colorlinks,
	citecolor=blue,
	linkcolor=red,
	urlcolor=blue}

\usepackage{url}            % simple URL typesetting
\usepackage{booktabs}       % professional-quality tables
\usepackage{amsfonts}       % blackboard math symbols
\usepackage{nicefrac}       % compact symbols for 1/2, etc.
\usepackage{microtype}      % microtypography
\usepackage{xcolor}         % colors

\usepackage{times}
\usepackage{epsfig}
\usepackage{graphicx}
\usepackage{tikz}
\usepackage{tikz-3dplot}
\usetikzlibrary{arrows}
\usetikzlibrary{decorations.markings}

\usetikzlibrary{positioning}

\usepackage{caption}
\usepackage{subcaption,mwe}

\usepackage{booktabs}
\usepackage{epigraph}

\usepackage{bm}
\usepackage{amsmath,amsthm,amssymb,rotating, mathrsfs}

\usepackage{footmisc}

\usepackage{mathtools}
\usepackage{cancel} % provides commands for striking out mathematical expressions.

\usepackage{xfrac} % sfrac
\usepackage{xparse} 

\DeclareMathAlphabet{\mathpzc}{OT1}{pzc}{m}{it}
\usepackage{enumitem}

\newcommand{\cD}{\mathcal{D}}

\newcommand{\cF}{\mathcal{F}}
\newcommand{\cG}{\mathcal{G}}

\newcommand{\cI}{\mathcal{I}}

\newcommand{\cK}{\mathcal{K}}

\newcommand{\cS}{\mathcal{S}}

\newcommand{\cW}{\mathcal{W}}

\newcommand{\bbE}{\mathbb{E}}

\newcommand{\bbR}{\mathbb{R}}

\newcommand{\bB}{\bm{B}}
\newcommand{\bC}{\bm{C}}

\newcommand{\bG}{\bm{G}}

\newcommand{\bK}{\bm{K}}

\newcommand{\bU}{\bm{U}}
\newcommand{\bV}{\bm{V}}
\newcommand{\bW}{\bm{W}}
\newcommand{\bX}{\bm{X}}
\newcommand{\bY}{\bm{Y}}

\newcommand{\ba}{\bm{a}}

\newcommand{\bd}{\bm{d}}

\newcommand{\bg}{\bm{g}}
\newcommand{\bh}{\bm{h}}

\newcommand{\bw}{\bm{w}}
\newcommand{\bx}{\bm{x}}
\newcommand{\by}{\bm{y}}

\newcommand{\argmin}{\mathop{\rm argmin}}

\newcommand{\myparagraph}[1]{\smallskip\noindent\textbf{#1.}}

\DeclareMathOperator{\ICL}{\texttt{ICL}}

\newtheorem{theorem}{Theorem}

\newtheorem{prop}{Proposition}

\theoremstyle{definition}
\newtheorem{definition}{Definition}
\newtheorem{example}{Example}

\newtheorem{assumption}{Assumption}

\theoremstyle{remark}
\newtheorem{remark}{Remark}

\newcommand{\ra}[1]{\renewcommand{\arraystretch}{#1}}
\usepackage[textsize=tiny]{todonotes}

\icmltitlerunning{The Ideal Continual Learner (ICL)}

\begin{document}

\twocolumn[
\icmltitle{The Ideal Continual Learner: An Agent That Never Forgets}



\icmlsetsymbol{equal}{*}

\begin{icmlauthorlist}
\icmlauthor{Liangzu Peng}{JHU,UPenn}
\icmlauthor{Paris V. Giampouras}{JHU}
\icmlauthor{Ren\'e Vidal}{UPenn,NORCE}
\end{icmlauthorlist}

\icmlaffiliation{JHU}{Mathematical Institute for Data Science, Johns Hopkins University, Baltimore, USA}
\icmlaffiliation{UPenn}{Innovation in Data Engineering and Science (IDEAS), University of Pennsylvania, Philadelphia, USA}
\icmlaffiliation{NORCE}{NORCE Norwegian Research Centre, Norway}

\icmlcorrespondingauthor{Liangzu Peng}{lpenn@seas.upenn.edu}

\icmlkeywords{Continual Learning, Catastrophic Forgetting, Generalization Bounds, Bilevel Optimization, Multitask Optimization, Constrained Learning, Stochastic Optimization with Expectation Constraints, Subspace Tracking, Streaming PCA, Matrix Factorization}

\vskip 0.3in
]



\printAffiliationsAndNotice{}  % leave blank if no need to mention equal contribution

\begin{abstract}
The goal of continual learning is to find a model that solves multiple learning tasks which are presented sequentially to the learner. A key challenge in this setting is that the learner may \textit{forget} how to solve a previous task when learning a new task, a phenomenon known as \textit{catastrophic forgetting}. To address this challenge, many practical  methods have been proposed, including memory-based, regularization-based, and expansion-based methods. However, a rigorous theoretical understanding of these methods remains elusive. This paper aims to bridge this gap between theory and practice by proposing a new continual learning framework called \textit{Ideal Continual Learner ($\ICL$)}, which is guaranteed to avoid catastrophic forgetting by construction. We show that $\ICL$ unifies multiple well-established continual learning methods and gives new theoretical insights into the strengths and weaknesses of these methods. We also derive generalization bounds for $\ICL$ which allow us to theoretically quantify \textit{how rehearsal affects generalization}. Finally, we connect $\ICL$ to several classic subjects and research topics of modern interest, which allows us to make historical remarks and inspire future directions.
\end{abstract}

% \vspace{-5mm}

% \epigraph{\textit{War is peace, freedom is slavery, ignorance is strength.} }{George Orwell, ``Nineteen Eighty-Four'' (1949) }

% \vspace{-5mm}

\section{Introduction}
The goal of building intelligent machines that are adaptive and self-improving over time has given rise to the \textit{continual learning} paradigm, where the \textit{learner} needs to solve multiple tasks presented sequentially \citep{Thrun-RAS1995}. In this setting, a key challenge known as \textit{catastrophic forgetting} \citep{Nccloskey-1989} is that, unlike humans, the agent may \textit{forget} how to solve past tasks (i.e., exhibit larger errors on past tasks), after learning the present one.
To address this problem, many practical methods have been proposed. For example, \textit{\textbf{memory-based}} methods store data from past tasks and use it to solve the present task \citep{Lopez-NeurIPS2017}, \textit{\textbf{regularization-based}} methods add some regularization terms to the loss in order to prevent overfitting to the current task \citep{Kirkpatrick-2017}, and \textit{\textbf{expansion-based}} methods expand the network architecture to accommodate learning a new task \citep{Rusu-arXiv2016}.

%This goal of continual learning without catastrophic forgetting has recently been approached via deep networks \citep{Lecun-2015}. 
% , as they are flexible in modeling complicated problems and efficient in making use of big data. 
%Behind the current deep continual learning practice are three high-level engineering ideas: (
%1) \textit{\textbf{memory-based}}, that is to store some past data for retraining or optimization with the present task \citep{Lopez-NeurIPS2017}; (2) \textit{\textbf{regularization-based}}, that is to add some regularization terms for preventing overfitting when optimizing the current task  \citep{Kirkpatrick-2017}; (3) \textit{\textbf{expansion-based}}, that is to expand deep networks and adjust the network architecture when learning a new task \citep{Rusu-arXiv2016}.

Deep continual learning methods that embody either of these three ideas, or combinations thereof, have greatly improved the empirical performance in resisting catastrophic forgetting \citep{Awesome-LL}.\footnote{We discuss the most relevant references in the main paper as we proceed, and we review related works in detail in the appendix.\label{footnote:review} } While these methods have evolved so rapidly, at least two questions remain under-explored:
\begin{enumerate}[wide,parsep=-1pt,topsep=0pt,label=(Q\arabic*)]
    \item \label{D1} \textit{Is there any \textbf{mathematical} connection between continual learning and other conceptually related fields}? The answer seems elusive even for \textit{multitask learning}  (i.e., joint training of all tasks): \citet{Chaudhry-NeurIPS2020} and \citet{Mirzadeh-ICLR2021}, among many others, advocated that performing multitask learning gives the best performance for continual learning, while \citet{Wu-arXiv2022} claimed the opposite.
    \item \label{D2} \textit{How does the rehearsal mechanism \textbf{provably} affect generalization}? \citet{Lopez-NeurIPS2017} claimed that \textit{rehearsal} (i.e., joint training with part of previous data stored in memory and the data of the present task) would result in overfitting (and jeopardize generalization). \citet{Chaudhry-arXiv2019v4} rebutted against this claim by empirical evidence, while \citet{Verwimp-ICCV2021} offered empirical counter-evidence for what \citet{Chaudhry-arXiv2019v4} argued. Indeed, the answer to \ref{D2} has been unclear, which is why \citet{Verwimp-ICCV2021} posed this as a serious open problem.   
%    it has been unclear how rehearsal \textit{provably} affects generalization, which is why \citet{Verwimp-ICCV2021} posed this as a serious open problem.
\end{enumerate}

% Debate \ref{D1} has arisen, as the \textit{mathematical} connection of continual learning to multitask learning (and other fields), provided that it exists, is seldom if ever spelled out. Debate \ref{D2} has surfaced, indicating that efforts are much needed to \textit{keep} our theoretical understanding of the rehearsal mechanism \textit{abreast of} its practical success in preventing forgetting.
% as there has been no unifying theoretical framework for continual learning that allows us to interrogate basic assumptions, discuss fundamental philosophy, and shed light on practical techniques.
\begin{figure}
    \centering
    \includegraphics[width=0.42\textwidth]{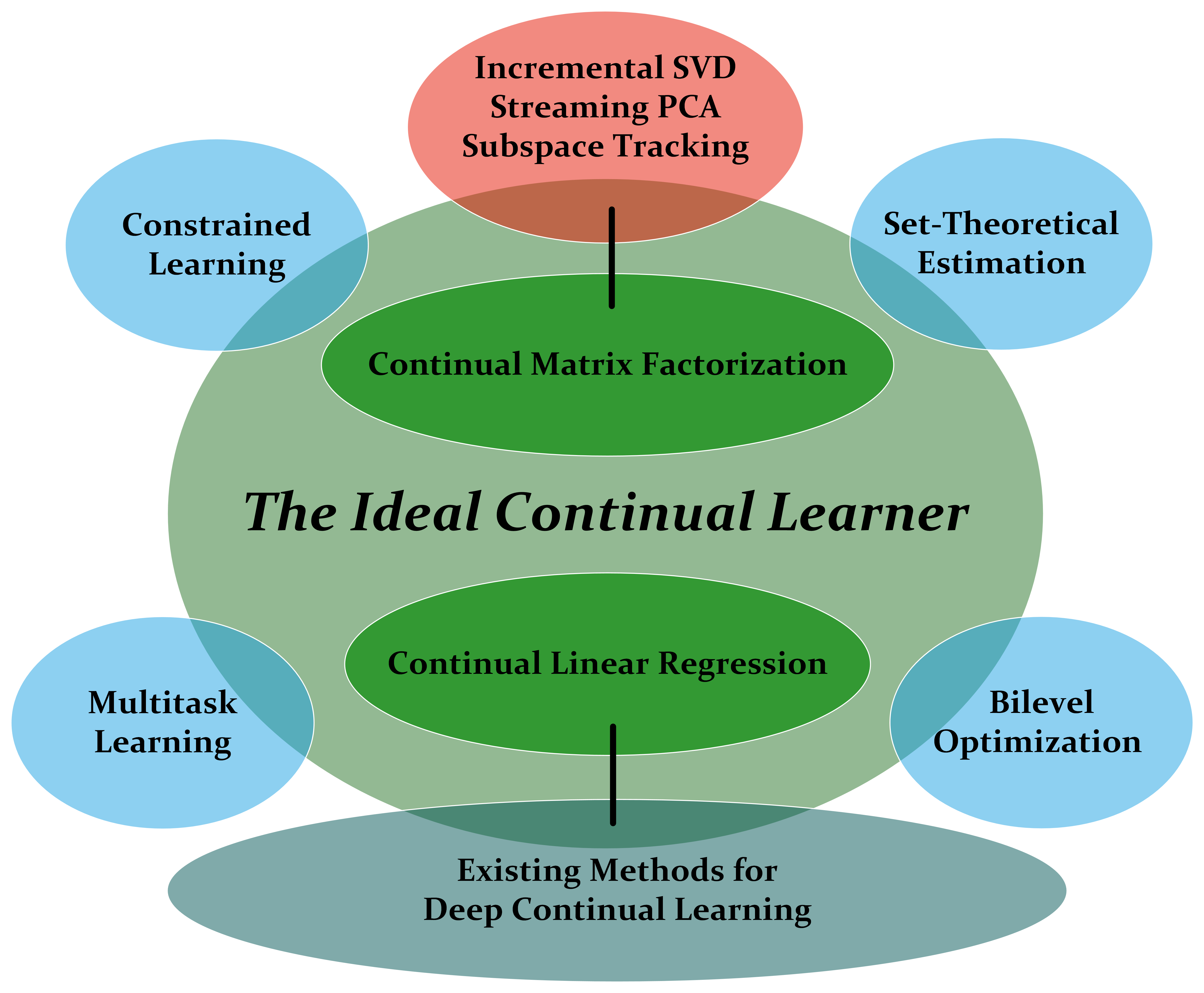}
    \caption{The \textit{Ideal Continual Learner} ($\ICL$) and related subjects. The exact connections between $\ICL$ and these subjects are discussed throughout the paper and appendix. }
    \label{fig:connection}
\end{figure}
\myparagraph{Our Contributions} In this paper, we focus on theoretically understanding continual learning and catastrophic forgetting by trying to answer Questions \ref{D1} and \ref{D2}. In particular:

\begin{table*}%[h!] 
\centering
	\ra{1.3}
	\begin{tabular}{lll}\toprule
	 %   Notations & Comments \\
     %   \midrule
        \S \ref{section:optimization} & \multicolumn{2}{c}{Optimization Basics of Continual Learning} \\
        & \S \ref{subsection:ICL-def} \& \S \ref{subsection:2principles} & Definition of $\ICL$ (Definition \ref{definition:CL-dagger}) \& $\ICL$ never forgets (Proposition \ref{prop:sufficiency}) \\
       % \S \ref{section:CLR-CMF} & & Examples of $\ICL$ \& insights into deep continual learning \\ 
        & \S \ref{subsection:bilevel-dagger} \& \S \ref{subsection:multitask-dagger} &  $\ICL$ is a bilevel optimizer (Propositions \ref{prop:inequality-constrained-CL-dagger} and \ref{prop:convex-differentiable}) \& $\ICL$ is a multitask learner (Proposition \ref{prop:multitask}) \\
        \midrule 
        \S \ref{section:CLR-CMF} & \multicolumn{2}{c}{Examples of $\ICL$ \& Understanding Deep Continual Learning } \\ 
         &  \S \ref{subsubsection:DCL-ICL-CLR} & $\ICL$ is a \textit{\textbf{memory-based (projection-based) method}} (Proposition \ref{proposition:OGD} and Remark \ref{remark:projection-methods}) \\ 
         & \S \ref{subsubsection:DCL-ICL-CMF} & $\ICL$ is an \textit{\textbf{expansion-based method}} \& wide neural networks forget less catastrophically \\
         \midrule 
        \S \ref{section:generalization} & \multicolumn{2}{c}{Generalization Basics of Continual Learning} \\
        &\S \ref{subsection:constrained-learning} & $\ICL$ and constrained learning \& $\ICL$ is a \textit{\textbf{regularization-based method}} \\
	&\S \ref{subsection:rehearsal} &  Generalization guarantees of rehearsal  (Theorem \ref{theorem:rehearsal}) \& remarks on memory selection methods \\ 
		\bottomrule
	\end{tabular}
	\caption{Structure and messages of the paper.} \label{table:messages}
\end{table*}

\begin{itemize}[wide,parsep=-1pt,topsep=0pt]
    \item We propose a general framework for continual learning, called the \textit{Ideal Continual Learner ($\ICL$)}, and we show that, under mild assumptions,  $\ICL$ \textit{never forgets}. This characterization of never forgetting makes it possible to address Questions \ref{D1} and \ref{D2} via dissecting the optimization and generalization properties of $\ICL$.

    \item Question \ref{D1} is considered throughout the paper: We bring to light the connection of $\ICL$ to many other fields---visualized in Figure \ref{fig:connection}---including \textit{set-theoretical estimation} (\S \ref{subsection:2principles}), \textit{bilevel optimization} (\S \ref{subsection:bilevel-dagger}), multitask learning (\S \ref{subsection:multitask-dagger}), deep continual learning (\S \ref{section:CLR-CMF}), \textit{incremental SVD} (\S \ref{subsection:CMF}), and \textit{constrained learning} (\S \ref{subsection:constrained-learner}). In particular, we dissect $\ICL$ in two examples, \textit{continual learning regression} (\S \ref{subsection:CLR}) and \textit{continual matrix factorization}  (\S \ref{subsection:CMF}), showing that $\ICL$ is a \textit{\textbf{memory-based optimization method}} for the former example and an \textit{\textbf{expansion-based}} for the latter. We also connect $\ICL$ to \textit{\textbf{regularization-based}} methods (\S \ref{subsection:constrained-learner}). These shed considerable lights on many aspects of existing deep continual learning methods, e.g., their failure cases, the role of memory and network widths, and so on. With all these connections, we provide historical context and novel insights for future research avenues.

    \item Question \ref{D2} is explored in \S \ref{section:generalization}, where we prove the generalization properties of $\ICL$ based on the classic statistical learning frameworks. Crucially, our theory quantifies how rehearsal influences generalization, shedding light on \ref{D2}.
\end{itemize} 
Since the paper consists of multiple messages relevant to continual learning, it might be beneficial to overview them in Table \ref{table:messages}, which includes direct links to the contents and aims at helping the reader navigate.

\section{Continual Learning Basics: Optimization}\label{section:optimization}
\subsection{Problem Setup}
Consider tasks $1,\dots,T$, where task $t$ can be solved by minimizing some problem of the form
\begin{align}\label{eq:task-t-dagger}
	\cG_t:=\argmin_{\bw \in \cW} L_t(\bw;D_t).
\end{align}
Here, $D_t$ is the given data set for task $t$, and $\cG_t$ is the set of global minimizers of the objective $L_t$ \eqref{eq:task-t-dagger}. By \textit{continual learning}, we mean solving the tasks sequentially from task $1$ to task $T$ and finding some ground truth in the \textit{hypothesis space} $\cW\subset \bbR^n$ that  minimizes all losses \eqref{eq:task-t-dagger}; if an algorithm can do so, then we say it never \textit{forgets}. While each task could have a different hypothesis space, we trade this generality for simplifying the presentation. 

For such ground truth to exist, we need a basic assumption:
\begin{assumption}[Shared Multitask Model$^\dagger$]\label{assumption:realizability-dagger}
	All tasks  \eqref{eq:task-t-dagger} share a common global minimizer, i.e., $\cap_{t=1}^T \cG_t \neq \varnothing$.
\end{assumption}

The intuition behind 
Assumption \ref{assumption:realizability-dagger}, which generalizes those of \citet{Evron-COLT2022} and \citet{Pengb-NeurIPS2022}, is that if $\cap_{t=1}^T \cG_t = \varnothing$ then there is no shared global minimizer and continual learning without forgetting is infeasible. Assumption \ref{assumption:realizability-dagger} might be relaxed into the existence of \textit{approximate} common global minimizers; we do not pursue this idea here.

Note that verifying whether $\cap_{t=1}^T \cG_t$ is empty or not is in general NP-hard even when every $\cG_t$ is a ``simple'' polytope \citep{Tiwary-DCG2008}, which is the main reason that \citet{Knoblauch-ICML2020} asserts continual learning without forgetting is in general NP-hard. However, Assumption \ref{assumption:realizability-dagger} sidesteps the curse of the NP-hardness, and  makes it possible to solve the continual learning problem computationally efficiently.%, as \citet{Evron-COLT2022,Pengb-NeurIPS2022} showed for linear regression and ``matrix factorization''. This gives the full story behind \ref{D2}.

\subsection{The Ideal Continual Learner$^\dagger$ (\texttt{ICL}$^\dagger$) }\label{subsection:ICL-def}
The main role of this section and the paper, is this:
\begin{definition}[The Ideal Continual Learner$^\dagger$,   $\ICL^\dagger$]\label{definition:CL-dagger} 
	With $\cK_0:=\cW$, $\ICL^\dagger$ is an algorithm  that solves the following program sequentially for  $t=1,2,\dots,T$:
	\begin{align}\label{eq:task-t-continual-dagger}
		\cK_t \gets\argmin_{\bw \in \cK_{t-1} } L_t(\bw;D_t).
	\end{align}
\end{definition}
The symbol $\dagger$ in Definition \ref{definition:CL-dagger} is a reminder that $\ICL^\dagger$ has not been proved (until \S \ref{section:generalization}) to be a learner in the \textit{statistical learning} sense. Recursion \eqref{eq:task-t-continual-dagger} asks for a lot: Computing the entire set  $\cK_1$ of global minimizers is hard enough, let alone constraining over it, recursively! Indeed, we use the word ``\textit{ideal}'' to imply that  $\ICL^\dagger$ is \textit{not realizable} at the current stage of research. However, take a leap of faith, and we will see many pleasant consequences of $\ICL^\dagger$ under Assumption \ref{assumption:realizability-dagger}, where the \textit{theoretical significance} of $\ICL^\dagger$ is treasured. 

An immediate observation is that $\cW=\cK_0\supset\cdots \supset \cK_T$, that is $\cK_t$ \textit{shrinks} (more precisely, does not \textit{grow}) over time. An analogy can be made from two perspectives. From a \textit{human learning} perspective, $\cK_t$ can be viewed as a \textit{knowledge representation}: After continually reading a 1000-page book, one internalizes the knowledge and represents the whole book with a few key notations or formulas, or with a single cheatsheet (but see also Remark \ref{remark:grow}, \textit{Grow} $=$ \textit{Shrink}). From a \textit{control} perspective, the analogy is that the \textit{uncertainty} $\cK_t$  over the true solutions reduces as $\ICL^\dagger$ learns from data and tasks. In fact, as we will review in Appendix \ref{section:set-theoretical-est}, $\ICL^\dagger$ is closely related to \textit{set-theoretical estimation} in control \cite{Combettes-IEEE1993,Kieffer-CDC1998}.

\subsection{\texttt{ICL}$^\dagger$ is Sufficient and Minimal}\label{subsection:2principles}
% Here we show that the Ideal Continual Learner$^\dagger$ ($\ICL^\dagger$) has two general and fundamental properties, \textit{sufficiency} \citep{Fisher-1922} and \textit{minimality} \citep{Lehmann-1950}.
In this section, we show that the Ideal Continual Learner$^\dagger$ is \textit{sufficient}  \citep{Fisher-1922} and \textit{minimal} \citep{Lehmann-1950}, which are two general and fundamental properties in the information-theoretical and statistical sense. 

\myparagraph{Sufficiency} It has been held that ``\textit{... preventing forgetting by design is therefore not possible [even if all losses $L_t$  \eqref{eq:task-t-dagger} are the same]}'' \citep{Van-NMI2022}. Nevertheless, under Assumption \ref{assumption:realizability-dagger}, $\ICL^\dagger$ never forgets by design: 
\begin{prop}[Sufficiency]\label{prop:sufficiency}
	Under Assumption \ref{assumption:realizability-dagger}, $\ICL^\dagger$ solves all tasks optimally. In other words, we have $\cK_t =\cap_{i=1}^t \cG_i$ for every $t=1,\dots, T$.
\end{prop}

\myparagraph{Minimality} Here we show (with rigor) that the knowledge representation $\cK_t$ is \textit{minimal} (given the order of the tasks). Consider the first two tasks, which are associated with objective functions $L_1$ and $L_2$ and global minimizers $\cG_1$ and $\cG_2$ respectively. Assume that $\cG_1$ and $\cG_2$ intersect (Assumption \ref{assumption:realizability-dagger}), and let $\hat{\bw}\in \cG_1\cap\cG_2$. A learner solves the first task by minimizing $L_1$ \eqref{eq:task-t-dagger}, stores some information $\cI_1$, and proceeds to the second (the stored information $\cI_1$ could consist of some data samples, gradients, global minimizers, and so on). If the stored information $\cI_1$ is not enough to reveal that $\hat{\bw}$ is optimal to task $1$, then it is impossible for the learner to figure out that $\hat{\bw}$ is actually simultaneously optimal to both tasks---even if it could find later that $\hat{\bw}$ is optimal to task $2$. Thus, either the learner would conclude  that no common global minimizer exists and Assumption \ref{assumption:realizability-dagger} is violated, or catastrophic forgetting inevitably takes place. In an independent effort, \citet{Pengb-NeurIPS2022} made a similar argument specifically for two-layer neural networks. Our argument is more general, and is for a different purpose.

% Suppose $\cG_2=\{\hat{\bw}\}$ (which means necessarily $\cG_1\cap\cG_2 = \{\hat{\bw} \}$). A learner solves the first task by minimizing $L_1$ \eqref{eq:task-t-dagger}, stores some information $\cI$, and proceeds for the second (the stored information $\cI$ could consist of some data samples, gradients, global minimizers, and so on). If the stored information $\cI$ is not enough to reveal that $\hat{\bw}$ is optimal to task $1$, then it is impossible for the learner to figure out that $\hat{\bw}$ is actually simultaneously optimal to both tasks---even if it could find later that $\hat{\bw}$ is optimal to task $2$. Thus, either the learner would conclude  that no common global minimizer exists and Assumption \ref{assumption:realizability-dagger} is violated, or catastrophic forgetting inevitably takes place. In an independent effort, \citet{Pengb-NeurIPS2022} made a similar argument specifically for two-layer neural networks. Our argument is more general and simpler, and is for a different purpose.

Note that $\hat{\bw}$ can be any element of $\cK_2=\cG_1\cap\cG_2$. Thus, storing a proper subset of  $\cK_2$ is sub-optimal since that might exclude global minimizers (e.g., $\hat{\bw}$) of subsequent tasks corresponding to $t>2$ and lead to catastrophic forgetting. In light of this, we observe that catastrophic forgetting can only be prevented if we store the entire set $\cK_t$ or its \textit{equivalent information} (we will soon see some equivalent representations of $\cK_t$). In this sense, we say the knowledge representation $\cK_t$ is minimal. It is the two properties, sufficiency and minimality, which $\ICL^\dagger$ (Definition \ref{definition:CL-dagger}) enjoys, that justify the naming: \textit{the} Ideal Continual Learner$^\dagger$ that never forgets.

\subsection{\texttt{ICL}$^\dagger$ $=$ Bilevel Optimizer}\label{subsection:bilevel-dagger}
Since the constraint $\bw\in\cK_{t-1}$ of \eqref{eq:task-t-continual-dagger} requires $\bw$ to be a global minimizer of all previous tasks,  for each task $\ICL^\dagger$ needs to solve a bilevel program \eqref{eq:task-t-continual-dagger}, which is in general difficult \citep{Vicente-1994,Jiang-arXiv2022}. However, with Assumption \ref{assumption:realizability-dagger}, we can describe \eqref{eq:task-t-continual-dagger} in a relatively simple way. The first description is immediate:

\begin{prop}\label{prop:inequality-constrained-CL-dagger}
    Under Assumption \ref{assumption:realizability-dagger}, for every $t=1,\dots,T$, the recursion \eqref{eq:task-t-continual-dagger} of $\ICL^\dagger$ is equivalent to 
	\begin{equation}\label{eq:inequality-constrained-dagger}
		\begin{split}
			&\ \min_{\bw \in\cW} L_t(\bw; D_t) \\
			\textnormal{s.t.} &\ L_i(\bw; D_i) \leq c_i,\ \forall i=1,\dots, t-1,
		\end{split}
	\end{equation}
	where $c_i$ is the minimum value  of $L_i(\bw; D_i)$ over $\cW$, computed during solving previous tasks.
\end{prop}
In Proposition \ref{prop:inequality-constrained-CL-dagger}, the inequality of $L_i(\bw; D_i) \leq c_i$ can be replaced by equality $=$, and the constraint of \eqref{eq:inequality-constrained-dagger} is understood as trivially fulfilled if $t=1$.
\begin{remark}[$\cK_t$ = Loss + Data]
    If we store all data and losses, we can recover $\cK_t$ via minimizing \eqref{eq:inequality-constrained-dagger}, hence (obviously) memorizing data and losses can prevent forgetting.
\end{remark}

A formulation similar to \eqref{eq:inequality-constrained-dagger} is known in the literature; see, e.g., \citet{Lopez-NeurIPS2017}. However, as in many continual learning papers, their formulation is dominated by computational considerations in the deep learning context, e.g., their $c_i$ is the loss of a sufficiently trained deep network for the $i$-th task over part of data samples; it is not necessarily a minimum value. Such approach, though practically important, makes it hard to derive theoretical properties.

%Different from the multitask formulation (Proposition \ref{prop:multitask}),  \eqref{eq:inequality-constrained-dagger} entails an extra cost of storing the minimum values $c_i$. Since \eqref{eq:inequality-constrained-dagger} has no \textit{strictly} feasible point, it does not fulfill the Slater’s condition. Hence, quoting \citet{Jiang-arXiv2022}, \eqref{eq:task-t-continual-dagger} can not be treated \textit{simply as a classic constrained optimization problem [\eqref{eq:inequality-constrained-dagger}] and calls for new theories and algorithms tailored to its hierarchical [bilevel] structure}. We believe this is an important viewpoint to be aware of, as continual learning has often been interpreted as constrained optimization since at least the seminar work of \citet{Lopez-NeurIPS2017}.

$\ICL^\dagger$ is also equivalent to the following formulation:
\begin{prop}\label{prop:convex-differentiable}
	Assume objective functions $L_1,\dots,L_T$ are convex and differentiable. Let $\cW=\bbR^n$. Under Assumption \ref{assumption:realizability-dagger}, each step \eqref{eq:task-t-continual-dagger} of $\ICL^\dagger$ is equivalent to 
	\begin{equation}\label{eq:task-t-continual-grad-dagger}
		\begin{split}
			&\ \min_{\bw \in \bbR^n} L_t(\bw; D_t) \\
			\textnormal{s.t.} &\ \nabla L_i(\bw; D_i) =0,\ \forall i=1,\dots, t-1,
		\end{split}
	\end{equation}
\end{prop}
% \begin{proof}
% 	It suffices to show that the constraints of \eqref{eq:task-t-continual-dagger} and \eqref{eq:task-t-continual-grad-dagger} are equivalent. Assumption \ref{assumption:realizability-dagger} implies $\cK_t=\cap_{i=1}^t\cG_i$, so
% 	\begin{equation}
% 		\begin{split}
% 			\bw \in \cK_t  &\Leftrightarrow \bw \in \cG_i,\  \forall i=1,\dots,t \\
% 			&\Leftrightarrow \nabla L_i(\bw; D_i) = 0,\  \forall i=1,\dots,t
% 		\end{split}
% 	\end{equation}
% 	The last equivalence is due to the assumptions on $L_i$'s.
% \end{proof}
\begin{remark}[$\cK_t$ $=$ Gradient Equations]
    If we store all equations $\nabla L_i(\bw; D_i) = 0$, we can recover $\cK_t$ via solving \eqref{eq:task-t-continual-grad-dagger}, so memorizing these equations resists forgetting.
\end{remark}
More generally, a common idea in bilevel programming is to rewrite the \textit{lower-level problem} (e.g., $\bw\in\cK_{t-1}$) as KKT conditions, and many algorithms exist for the latter formulation; see, e.g., \citet{Dempe-MP2012} and the follow-up works. While it is beyond the scope of the paper, contextualizing these ideas for implementing $\ICL^\dagger$ \eqref{eq:task-t-continual-dagger} is an important direction that would inspire novel and theoretically grounded continual learning algorithms. Note that bilevel programming has been used for continual learning \citep{Borsos-NeurIPS2020}, but the use is mainly for selecting which data samples to store in the memory.

% Proposition \ref{prop:convex-differentiable} can be extended to the case where $\cW$ is a \textit{proper} semi-algebraic subset of $\bbR^n$ by leveraging the sufficiency and necessity of the KKT conditions for optimality under a strong duality assumption; we omit the details here.

%We extend Proposition \ref{prop:convex-differentiable} to the case where $\cW\subsetneq\bbR^n$:
%\begin{prop}[$\cK_t$ = KKT Conditions]
%	Let $f_1,\dots,f_M:\bbR^n\to \bbR$ be differentiable functions and assume 
%	\begin{align}
%		\cW= \{ \bw\in\bbR^n:  f_i(\bw)\leq 0, \forall i=1,\dots, M_{}\}.
%	\end{align}
%	Assume objective functions $L_1,\dots,L_T$ are differentiable, and strong duality holds for the optimization problem of each task $t$ \eqref{eq:task-t-dagger} and their dual program. Under Assumption \ref{assumption:realizability-dagger}, $\bw\in \cK_t$ if and only if the following KKT conditions 
%	\begin{equation}
%		\begin{split}
%			\nabla
%		\end{split}
%	\end{equation}
%	hold for some dual (optimal) variables  $\lambda_1,\dots,\lambda_M\geq 0$.
%\end{prop}
%\begin{remark}[$\cK_t$ = KKT Conditions]
%\end{remark}

\subsection{\texttt{ICL}$^\dagger$ $=$ Multitask Learner$^\dagger$}\label{subsection:multitask-dagger}
Under Assumption \ref{assumption:realizability-dagger}, the following connection between continual learning and multitask learning arises naturally:
\begin{prop}[$\ICL^\dagger$ = Multitask Learner$^\dagger$]\label{prop:multitask}
	Let $\alpha_1,\dots,\alpha_t$ be arbitrary positive numbers. Recall $\cK_t$ is defined in \eqref{eq:task-t-continual-dagger} as the output of $\ICL^\dagger$. Under Assumption \ref{assumption:realizability-dagger}, we have
	\begin{align}\label{eq:CL=multitask}
		\cK_t = \argmin_{w\in \cW} \sum_{i=1}^t  \alpha_i L_i(\bw; D_i), \ \ \ \forall t=1,\dots,T.
	\end{align}
\end{prop}

Assumption \ref{assumption:realizability-dagger} implies that  $\ICL^\dagger$ never forgets and provides the best performance for continual learning, and that minimizing the multitask loss \eqref{eq:CL=multitask} recovers $\ICL^\dagger$. Hence, under Assumption \ref{assumption:realizability-dagger}, minimizing the multitask loss yields the best performance for continual learning. This simple result addresses issues pertaining to Question \ref{D1}.

\section{Examples of \texttt{ICL}$^\dagger$}\label{section:CLR-CMF}
This section discusses $\ICL^\dagger$ for two examples, \textit{continual linear regression} (\S \ref{subsection:CLR}), \textit{continual matrix factorization} (\S \ref{subsection:CMF}). By way of examples, we acquire some understanding of existing deep continual learning practices (\S \ref{subsubsection:DCL-ICL-CLR}, \S \ref{subsubsection:DCL-ICL-CMF}).

\subsection{Continual Linear Regression}\label{subsection:CLR}
In \S \ref{subsubsection:background-CLR}, we review the problem setup of continual linear regression \citep{Evron-COLT2022}. In \S \ref{subsubsection:ICL-CLR} we present the implementation of $\ICL^\dagger$ for continual linear regression. In \S \ref{subsubsection:DCL-ICL-CLR} we draw connections to deep continual learning, shedding light on \textit{\textbf{memory-based optimization methods}}.

\subsubsection{Background and Problem Setup}\label{subsubsection:background-CLR} 
In \textit{continual} linear regression, the data $D_t:=(\bX_t,\by_t)$ of each task $t$ consists of a feature matrix $\bX_t\in\bbR^{m_t\times n}$ and a response vector $\by_t\in\bbR^{m_t}$; here $m_t$ is the number of samples of task $t$ and $n$ their dimension, and the loss $L_t$ is 
\begin{align}\label{eq:task-CLR}
	L_t(\bw; (\bX_t,\by_t)) = \| \bX_t \bw - \by_t \|_2^2.
\end{align}

The set $\cG_t$ of global minimizers for each task \eqref{eq:task-CLR} is exactly an affine subspace. Under Assumption \ref{assumption:realizability-dagger}, the intersection of these affine subspaces is not empty, and finding one element in that intersection would solve all tasks optimally. Note that, if $\bX_t$ is of full column rank, then solving task $t$ yields a unique solution, which is optimal to all tasks by Assumption \ref{assumption:realizability-dagger}. To avoid this trivial case, we follow \citet{Evron-COLT2022} and assume $\bX_t$ is rank-deficient for every $t$. Studying continual regression is of interest as (continual) regression is closely related to deep (continual) learning in the \textit{neural tangent kernel} regime \citep{Jacot-NeurIPS2018}. Moreover, as will be shown in \S \ref{subsubsection:DCL-ICL-CLR}, $\ICL^\dagger$ connects continual linear regression to deep continual learning even more tightly. 

% \citet{Evron-COLT2022} employed stochastic gradient descent to minimize each task $t$ \eqref{eq:task-CLR}, initialized at the estimate of previous task.

With some initialization, the algorithm of \citet{Evron-COLT2022} proceeds in a successive (or alternating) projection fashion (Figure \ref{fig:CLR-Evron}). When presented with task $t$, it projects onto the affine subspace $\cG_t=\{\bw\in \bbR^n: \bX_t^\top\bX_t \bw =\bX_t^\top \by_t  \}$; this projection can be implemented via invoking a singular value decomposition (SVD) or leveraging the implicit bias of (stochastic) gradient descent. \citet{Evron-COLT2022} employed the latter approach as it is \textit{memoryless}, i.e., it does not use any extra storage except the estimate and current data. However, being memoryless might entail catastrophic forgetting in the worst case \citep[Theorem 7]{Evron-COLT2022}.

% \begin{remark}[Memoryless $=$ Forgetting]\label{remark:memoryless=forgetting}
% 	The method of \citet{Evron-COLT2022} is \textit{memoryless}, i.e., it does not use any extra storage except the estimate and current data. However, being memoryless might entail catastrophic forgetting in the worst case \citep[Thm 7]{Evron-COLT2022}.
% \end{remark}

\tikzset{
  big arrow/.style={
    decoration={markings,mark=at position 1 with {\arrow[scale=0.9,#1]{>}}},
    postaction={decorate},
    shorten >=0.3pt},
  big arrow/.default=blue}
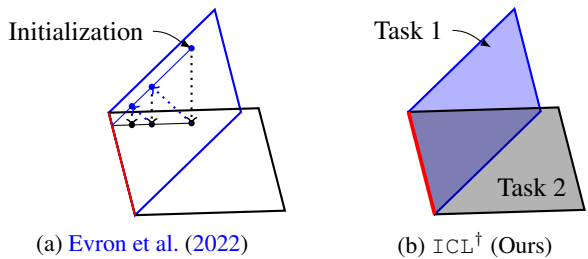
\begin{figure}[]
	\centering
	% Generat amb TikZ
	\tdplotsetmaincoords{70}{85}
	\subfloat[\citet{Evron-COLT2022}]{
		\begin{tikzpicture}[tdplot_main_coords]
			
			\draw[thick,blue] (-2,1.414,1.414) -- (-2,0,0) -- (2,0,0) -- (2,1.414,1.414) -- cycle;
			
			\draw[thick] (-2,0,0) -- (-2,2,0) -- (2,2,0) -- (2,0,0) -- cycle;
			\draw[thick,red](-2,0,0)--(2,0,0);
			
			\foreach \x in {0,...,2}
			\fill[blue] (-1.5,1.0605/2^\x,1.0605/2^\x) circle (1.2pt);
			\foreach \x in {0,...,2}
			\fill[black] (-1.5,1.0605/2^\x,0) circle (1.2pt);
			\foreach \x in {0,...,2}
			\draw[thick,big arrow=black,black,dotted] (-1.5,1.0605/2^\x,1.0605/2^\x) -- (-1.5,1.0605/2^\x,0);
			\foreach \x in {0,...,1}
			\draw[thick,big arrow=blue,blue,dotted] (-1.5,1.0605/2^\x,0) -- (-1.5,1.0605/2^\x/2,1.0605/2^\x/2);
			
			\node[anchor= east] (line) at (-1,0.5,1.5) {Initialization};
			\draw[-latex] (line) to[out=0,in=150] (-1.5,1.0605,1.0605);

                \draw[blue] (-1.5,1.0605,1.0605) -- (-1.5,0,0);
                \draw[black] (-1.5,1.0605,0) -- (-1.5,0,0);
		\end{tikzpicture}
		\label{fig:CLR-Evron}
	}
	\ \ \ \ \ \ \ \ \ 
	\subfloat[$\ICL^\dagger$ (Ours)]{
		\begin{tikzpicture}[tdplot_main_coords]
			
			\draw[thick,blue] (-2,1.414,1.414) -- (-2,0,0) -- (2,0,0) -- (2,1.414,1.414) -- cycle;
			\fill[blue,opacity=0.3] (-2,1.414,1.414) -- (-2,0,0) -- (2,0,0) -- (2,1.414,1.414) -- cycle;
			
			\draw[thick] (-2,0,0) -- (-2,2,0) -- (2,2,0) -- (2,0,0) -- cycle;
			\fill[black,opacity=0.3] (-2,0,0) -- (-2,2,0) -- (2,2,0) -- (2,0,0) -- cycle;
			\draw[line width=0.5mm,red](-2,0,0)--(2,0,0);
			
			\node[anchor= east] (line) at (-1,0.5,1.5) {Task 1};
			\draw[-latex] (line) to[out=0,in=150] (-1.5,1.0605,1.0605);
			
			\node[anchor= west] (line) at (1,0.8,0) {Task 2};
			%\draw[-latex] (line) to[out=0,in=30] (1.5,1.0605,0);
		\end{tikzpicture}
		\label{fig:CLR-Ours}
	}
        \caption{For two tasks, \citet{Evron-COLT2022} follow the arrows (\ref{fig:CLR-Evron}) and project the current estimate onto the affine subspace of solutions (\textcolor{blue}{blue} or black), in an alternating fashion, eventually reaching a common solution (\textcolor{red}{red}). Instead, $\ICL^\dagger$ solves task 1, stores the affine subspace (\ref{fig:CLR-Ours}, \textcolor{blue}{blue}), uses it to regularize task 2 and finds all solutions (\ref{fig:CLR-Ours}, \textcolor{red}{red}).}
	\label{fig:CLR}
\end{figure}

\subsubsection{Implementing \texttt{ICL}$^\dagger$}\label{subsubsection:ICL-CLR}
Our $\ICL^\dagger$ method is illustrated in Figure \ref{fig:CLR-Ours}. As per \eqref{eq:task-t-continual-dagger}, $\ICL^\dagger$ sets $\cW\gets\bbR^n$, finds the affine subspace $\cG_1$ of task 1, uses it to regularize task 2, finds their common solutions $\cG_1\cap \cG_2$, and so forth. We describe some more details next.

Let $\hat{\bw}_t\in \cK_t$ be a common global minimizer to tasks $1,\dots,t$. Let $\bK_t$ be an orthonormal basis matrix\footnote{To simplify the presentation, we did not annotate some matrix sizes or subspace dimensions. It is understood that such matrices are of suitable sizes and subspaces of appropriate dimensions. \label{footnote:matsize} } for the intersection of the nullspaces of $\bX_1,\dots,\bX_t$. Note that $\bw\in \cK_t$ if and only if $\bw$ can be written as $\bw=\hat{\bw}_t + \bK_t \ba$ for some coefficient vector $\ba$, so for implementing $\ICL^\dagger$ we will compute and store  $(\hat{\bw}_t,\bK_t)$. Such computation is summarized in the following proposition:
\begin{prop}\label{prop:implement-ICL-CLR}
    $\ICL^\dagger$ for continual linear regression can be implemented as follows. Given data $(\bX_1,\by_1)$, one can compute $(\hat{\bw}_1,\bK_1)$ via  SVD. With $t>1$, given $(\hat{\bw}_{t-1},\bK_{t-1})$ and data $(\bX_t,\by_t)$, one can compute $(\hat{\bw}_t,\bK_t)$ via solving
\begin{align}
	     &\min_{\bw\in\cK_{t-1}}  \| \bX_t \bw - \by_t \|_2^2  \label{eq:continual-CLR} \\
  \Leftrightarrow & \min_{\ba} \| \bX_t (\hat{\bw}_{t-1} + \bK_{t-1} \ba) - \by_t \|_2^2  \label{eq:continual-CLR-unconstrained}
\end{align}
In particular, \eqref{eq:continual-CLR-unconstrained} can be solved via SVD.\footnote{The use of SVD is to find singular vectors corresponding to the extreme singular values (in particular, zero singular values) of a given matrix. A basic fact from numerical linear algebra is that, in general, the extreme singular vectors of a given matrix can only be found iteratively and therefore inexactly \citep[Lecture 25]{Trefethen-1997}. Thus, the implementations that we suggested for continual linear regression (\S \ref{subsubsection:ICL-CLR}) and continual matrix factorization (\S \ref{subsubsection:ICL-CMF}) only \textit{approximately} implement $\ICL^\dagger$; note though that such approximation quality is typically very high due to the existence of industry-strength SVD algorithms. \label{footnote:svd} }
\end{prop}

% If $t=1$, we can compute $(\hat{\bw}_1,\bK_1)$ by performing an SVD on $\bX_1$, followed by a few matrix-vector multiplications. For $t>1$, we are given $(\bw_{t-1},\bK_{t-1})$ and data $(\bX_t,\by_t)$, from which $(\hat{\bw}_t,\bK_t)$ can be computed via 
% \begin{align}
% 	     &\min_{\bw\in\cK_{t-1}}  \| \bX_t \bw - \by_t \|_2^2  \label{eq:continual-CLR} \\
%   \Leftrightarrow & \min_{\ba} \| \bX_t (\hat{\bw}_{t-1} + \bK_{t-1} \ba) - \by_t \|_2^2  \label{eq:continual-CLR-unconstrained}
% \end{align}

% The normal equations of \eqref{eq:continual-CLR} are given by ($\overline{\bX}_t:= \bX_t \bK_{t-1}$)
% \begin{align*}
% 	\overline{\bX}_t^\top \overline{\bX}_t \ba = \overline{\bX}_t^\top (\by_t -  \bX_t \hat{\bw}_{t-1}).
% \end{align*}
% We can compute $(\hat{\bw}_t,\bK_t)$ for $t>1$ by a matrix multiplication $\overline{\bX}_t=\bX_t \bK_{t-1}$, an SVD on $\overline{\bX}_t$, and a few matrix-vector multiplications. 

Compared to the method of \citet{Evron-COLT2022}, $\ICL^\dagger$ uses some extra storage for $\bK_t$ to attain optimality and resist catastrophic forgetting; note though that the storage consumption never grows as $\cK_T\subset\cdots \subset \cK_1$. Finally, note that for continual linear regression, one can formulate $\ICL^\dagger$ in different ways than \eqref{eq:continual-CLR-unconstrained}; see, e.g., Proposition 5.5 of \cite{Evron-ICML2023} (based on Fisher information) and Proposition \ref{prop:convex-differentiable} (based on the first-order optimality conditions).

\subsubsection{Connections to Deep Continual Learning}\label{subsubsection:DCL-ICL-CLR}
We first connect $\ICL^\dagger$ to the \textit{orthogonal gradient descent} method (OGD) of \citet{Farajtabar-AISTATS2020}: 
\begin{prop}[Informal]\label{proposition:OGD}
When applied to continual linear regression, the OGD method consists of the updates
\begin{align}\label{eq:OGD-CLR}
		\bw^+ \gets \bw - \gamma \bK_{t-1}\bK_{t-1}^\top \bh,
\end{align}
where $\bh$ is the gradient of the objective of \eqref{eq:continual-CLR}, $\gamma$ stepsize, and $\bw$ (resp. $\bw^+$) the previous (resp. current) iterate. Moreover, OGD converges to a global minimizer of \eqref{eq:continual-CLR}.
\end{prop}
% \begin{proof}[Proof Sketch (Informal)]
% 	Gradient descent applied to  \eqref{eq:continual-CLR} is
% 	\begin{align*}%\label{eq:GD-CLR}
% 		\ba^+ \gets \ba - \gamma \bg, \ \bg := 2\overline{\bX}_t^\top (\overline{\bX}_t \ba +  \bX_t \hat{\bw}_{t-1} - \by_t ),
% 	\end{align*}
% 	where $\bg$ is the gradient of \eqref{eq:continual-CLR}, $\gamma$ stepsize, and $\ba$ ($\ba^+$) the previous and current iterate, respectively. Define $\bw^+:= \hat{\bw}_{t-1}+ \bK_{t-1}\ba^+$, and $\bw:= \hat{\bw}_{t-1}+ \bK_{t-1}\ba$. With some algebraic manipulation (omitted here), we see that the gradient update of $\ba$  updates the parameter $\bw$ \textit{implicitly}:
% 	\begin{align}\label{eq:OGD-CLR}
% 		\bw^+ \gets \bw - \gamma \bK_{t-1}\bK_{t-1}^\top \bh, 
% 	\end{align}
%         Here $\bh:=  2 \bX_t^\top (\bX_t \bw - \by_t)$ is the gradient of \eqref{eq:task-CLR}. Note that \eqref{eq:OGD-CLR} is exactly the formula of OGD applied to \eqref{eq:task-CLR}.
% \end{proof}
We prove Proposition \ref{proposition:OGD} in Appendix \ref{section:OGD-proof}, where we also review the OGD method; see \citet{Bennani-arXiv2020,Doan-AISTATS2021} for different theoretical aspects of OGD. % The phrase ``approximately'' of Proposition \ref{proposition:OGD} reflects the fact that,  in finite iterations, iterative methods such as \eqref{eq:OGD-CLR} can only find an \textit{approximate}, not \textit{exact}, minimizer. 

Since regression can be viewed as a single-layer linear network with a least-squares loss, formula \eqref{eq:OGD-CLR} can be extended into deep continual learning in a \textit{layer-wise} manner. This viewpoint allows us to make the following connection:
\begin{remark}\label{remark:projection-methods}
    For continual linear regression, the methods of \citet{Zeng-NMI2019,Saha-ICLR2021,Wang-CVPR2021,Kong-ECCV2022} are of the form \eqref{eq:OGD-CLR}, and so they all converge ``approximately'' to global minimizers of \eqref{eq:continual-CLR}; these methods differ mainly in how the projection $\bK_{t-1}\bK_{t-1}^\top$ is approximated and stored. For deep continual learning, their methods perform \eqref{eq:OGD-CLR} in a \textit{layer-wise} manner, with different projections to update the parameters of every linear layer.
\end{remark}
We review the methods of Remark \ref{remark:projection-methods} in Appendix \ref{subsection:CL-review-projection}. By showing that these methods can be derived from the principled objective \eqref{eq:continual-CLR} of $\ICL^\dagger$, we make them more interpretable. For example, \citet{Kong-ECCV2022} approximate $\bK_{t-1}\bK_{t-1}^\top$ better than \citet{Zeng-NMI2019}; not surprisingly, \citet{Kong-ECCV2022} gets better performance (see their Table 1). Moreover, we provide an important failure case. Note that all these methods (including OGD) approximately solve \eqref{eq:continual-CLR} and the correctness of \eqref{eq:continual-CLR} relies on Assumption \ref{assumption:realizability-dagger}. As a consequence, in the absence of Assumption \ref{assumption:realizability-dagger}, all these methods (and $\ICL^\dagger$) might be \textit{obsessed with the past}:

% Our contribution is pointing out that these \textit{modified} gradient descent methods \eqref{eq:OGD-CLR} actually execute \textit{vanilla} gradient descent, yet on a different objective \eqref{eq:continual-CLR}, which is obtained from the principle of $\ICL^\dagger$; in so doing, we make their methods more interpretable. For example, \citet{Kong-ECCV2022} approximate $\bK_{t-1}\bK_{t-1}^\top$ better than \citet{Zeng-NMI2019}; not surprisingly, \citet{Kong-ECCV2022} gets better performance (see their Table 1). Moreover, we provide an important failure case. All these methods (including OGD) approximately and implicitly solve \eqref{eq:continual-CLR} and the correctness of \eqref{eq:continual-CLR} relies on Assumption \ref{assumption:realizability-dagger}. As a consequence, in the absence of Assumption \ref{assumption:realizability-dagger}, all these methods (and $\ICL^\dagger$) might be \textit{obsessed with the past}:

\begin{example}[Past $=$ Present]\label{example:insufficiency}
Consider the continual linear regression problem \eqref{eq:task-CLR}, and assume $\cG_1$ and $\cG_2$ are two parallel and non-intersecting affine subspaces. In this situation, $\ICL^\dagger$ will compute $\cK_1=\cG_1$ and $\cK_2=\cG_1$. In other words, it will get stuck to optimal points of task 1, contained in $\cK_1$, failing to track optimal points of task 2.
\end{example}

% Our theoretical study here can be extended to kernel regression or deep networks in the \textit{neural tangent kernel} regime  \citep{Jacot-NeurIPS2018}. Indeed, \citet{Bennani-arXiv2020} analyzed continual kernel regression specifically for OGD; e.g., they showed that OGD does not alter the losses on previous kernel regression tasks (see also \citet{Doan-AISTATS2021}). This can also be seen in our context, as \eqref{eq:OGD-CLR} suggests: If $\bw\in\ \cK_{t-1}$ is optimal for previous tasks, then so is $\bw^+$. 

\subsection{Continual Matrix Factorization}\label{subsection:CMF}
The structure of this section parallels that of \S \ref{subsection:CLR}. In \S \ref{subsubsection:background-CMF}. We introduce the problem of \textit{continual matrix factorization}; this problem arises as a generalization of \citet{Pengb-NeurIPS2022}. In \S \ref{subsubsection:ICL-CMF}, we present the implementation of $\ICL^\dagger$ for this problem. In \S \ref{subsubsection:DCL-ICL-CMF}, we make connections to deep continual learning, highlighting  \textit{\textbf{expansion-based methods}}.

\subsubsection{Background and Problem Setup}\label{subsubsection:background-CMF} 
The \textit{continual} matrix factorization setting is as follows. The data $D_t:=\bY_t$ of task $t$ is a matrix $\bY_t$, and every column of $\bY_t$ lies in some (linear) subspace $\cS_t$ of $\bbR^n$; we assume the columns of $\bY_t$ span $\cS_t$ without loss of generality. Each task $t$ consists of factorizing $\bY_t$ into two matrices $\bU$ and $\bC$ such that $\bU \bC = \bY_t$. This corresponds to minimizing
\begin{align}\label{eq:loss-CMF}
	L_t\big((\bU,\bC); \bY_t \big) = \| \bU \bC -\bY_t  \|_\textnormal{F}^2,
\end{align}
which is a \textit{matrix factorization} problem. With the identity matrix $I$,\footref{footnote:matsize}, we assume $\bU^\top \bU=I$. The goal of \textit{continual} matrix factorization is to factorize the whole data matrix $[\bY_1\ \cdots\ \bY_T]$ into $\bU$ and $\bC$ such that $[\bY_1\ \cdots\ \bY_T]=\bU\bC$, under the constraint that $\bY_t$ is presented sequentially. Under certain conditions, matrix factorization \eqref{eq:loss-CMF} is equivalent to two-layer linear neural networks, with $\bC$ being the first layer of weight parameters and $\bU$ the second \citep{Baldi-NN1989,Vidal-DeepMath2020,Pengb-NeurIPS2022}. In that sense, analyzing continual matrix factorization would facilitate understanding deep continual learning.

\citet{Pengb-NeurIPS2022} performed one such analysis, based on prior works on orthogonal gradient descent \citep{Farajtabar-AISTATS2020,Chaudhry-NeurIPS2020} and matrix factorization \citep{Ye-NeurIPS2021}. They assumed that the rank $r$ of $[\bY_1\ \cdots\ \bY_T]$ is known, and that each $\bY_t$ is a vector. With $r$ given, \citet{Pengb-NeurIPS2022} maintain a basis matrix $\bU$ with $r$ columns throughout the learning process. Their method is not memory-efficient for two reasons: (i) Storing $r$ columns is unnecessary when the learner has not encountered the $r$-th sample (this extra consumption of memory is negligible for small $r$, though); (ii) their algorithm furthermore requires storing two projection matrices. In \S \ref{subsubsection:ICL-CMF}, we will show that $\ICL^\dagger$ can handle the more general situation where the rank $r$ is unknown, and it also overcomes the memory inefficiency of \citet{Pengb-NeurIPS2022}.

\subsubsection{Implementing \texttt{ICL}$^\dagger$}\label{subsubsection:ICL-CMF} 
To fully understand  \texttt{ICL}$^\dagger$ for continual matrix factorization, we first establish a basic geometric understanding of the problem and then discuss how to store $\cK_t$ efficiently, and finally, we describe the implementation details.

\myparagraph{Basic Subspace Geometry} Minimizing \eqref{eq:loss-CMF} in variable $\bC$ with $\bU$ fixed reveals that the optimal $\bC$ is exactly $\bU^\top\bY_t$, so \eqref{eq:loss-CMF} is equivalent to $ \| \bU \bU^\top \bY_t -\bY_t  \|_\textnormal{F}^2$, an objective function for \textit{principal component analysis} (PCA). In this way, continual matrix factorization relates to \textit{streaming PCA} \citep{Mitliagkas-NeurIPS2013,Pengb-NeurIPS2022}, \textit{incremental SVD} \citep{Bunch-1978}, and \textit{subspace tracking} \citep{Balzano-IEEE2018}. We review these subjects in Appendix \ref{section:streaming-PCA} to highlight the connection to $\ICL^\dagger$.

Geometrically, every orthonormal basis of any subspace\footref{footnote:matsize} containing $\cS_t$ is a global minimizer of this PCA objective. In other words, any global minimizer of \eqref{eq:loss-CMF} is of the form $(\bU, \bU^\top \bY_t)$, where $\bU$ is orthonormal with its range space $\text{range}(\bU)$ containing $\cS_t$. Ignoring the role of $\bU^\top \bY_t$ for simplicity, we can write the set $\cG_t$ of global minimizers as
\begin{align*}%\label{eq:CMF-Gt}
    \cG_t:= \big\{ \bU: \bU^\top \bU = I,\ \cS_t\subset \text{range}(\bU) \subsetneq \bbR^n \big\}; 
\end{align*}
we ruled out the case $\text{range}(\bU)=\bbR^n$ as this indicates trivial solutions. Assumption \ref{assumption:realizability-dagger} implies the intersection
\begin{align*}
    \cap_{i=1}^t \cG_i=\{ \bU: \bU^\top \bU = I, \sum\nolimits_{i=1}^t  \cS_i\subset \text{range}(\bU) \subsetneq \bbR^n \}
\end{align*}
is non-empty, which implies $\mathrm{dim}(\sum\nolimits_{t=1}^T \cS_t)<n$ or equivalent that the data matrix $[\bY_1\ \cdots\ \bY_T]$ is rank-deficient. Note that this is a weaker assumption than that of \citet{Pengb-NeurIPS2022}, who assumed the rank of $[\bY_1\ \cdots\ \bY_T]$ is given.

\myparagraph{Storing $\cK_t$} To implement $\ICL^\dagger$ under Assumption \ref{assumption:realizability-dagger}, we need to store $\cK_t=\cap_{i=1}^t \cG_i$ in a memory-efficient manner. Since we know that any orthonormal matrix whose range space contains $ \sum_{i=1}^t\cS_i$ is an element of $\cK_t$ (and vice versa), storing $\cK_t$ is equivalent to storing the subspace sum $\sum_{i=1}^t\cS_i$. Indeed, storing $\sum_{i}^t\cS_i$ gives enough information about $\cK_t$, which prevents forgetting (recall the minimality of $\cK_t$, \S \ref{subsection:2principles}). With this viewpoint, we will maintain an orthonormal basis matrix $K_t$ of $\sum_{i=1}^t \cS_i$ to implement $\ICL^\dagger$, where each subspace $\cS_i$ will be learned from data $\bY_i$. 

\myparagraph{Implementation Details} For $t=1$, we can compute the orthonormal basis $K_1$ of $\cS_1$ via an SVD\footref{footnote:svd} on $\bY_1$; e.g., set $K_1$ to be the matrix whose columns are left singular vectors of $\bY_1$ corresponding to its non-zero singular values.

For $t>1$, we need to compute $K_t$ from $\bY_t$ and $K_{t-1}$, under the inductive assumption that $K_{t-1}$ is an orthonormal basis matrix of $\sum_{i=1}^{t-1}\cS_i$. Since $\cK_{t-1}$ consists of orthonormal matrices $\bU$ whose range spaces contain $\text{range}(K_{t-1})$, we know that $\bU\in \cK_{t-1}$ if and only if $\bU^\top \bU=I$ and $\bU$ is of the form $[K_{t-1}\ U_t ]$ (up to some isometry). As such, the recursion \eqref{eq:task-t-continual-dagger} of $\ICL^\dagger$ is equivalent to
\begin{equation*}
    \begin{split}
        &\ \argmin_{\bU \in \cK_{t-1} } \big\| \bU \bU^\top \bY_t - \bY_t \big\|_{\textnormal{F} }^2 \\
    \Leftrightarrow &\ \overline{U}_t\in \argmin_{U_t} \big\| [K_{t-1}\ \  U_t]\ [K_{t-1}\ \  U_t]^\top \bY_t - \bY_{t} \big\|_{\textnormal{F} }^2 \\
    \Leftrightarrow &\ \overline{U}_t\in \argmin_{U_t} \big\| U_tU_t^\top Y_t - Y_t  \big\|_\textnormal{F} 
    \end{split}
\end{equation*}
where we defined $Y_t:= (I - K_{t-1}K_{t-1}^\top)\bY_t$ and used the fact  $K_{t-1}^\top U_t=0$. Similarly to the case $t=1$, $\overline{U}_t$ can be computed via an SVD on $Y_t$, and the rank of $Y_t$ determines the number of its columns. Setting $K_t\gets [K_{t-1}\ \overline{U}_t]$ furnishes a desired orthonormal basis  for $\sum_{i=1}^t \cS_i$.

\begin{remark}[New Knowledge $=$ New Parameters]
    If $\cS_t$ is contained in $\sum_{i=1}^{t-1} \cS_i$, then there is no new ``knowledge'' to learn and $Y_t=0$; in this case we set $K_t\gets K_{t-1}$. The amount of new knowledge is encoded as the rank of $Y_t$, which determines the number of new parameters to add.
\end{remark}

\subsubsection{Connections to Deep Continual Learning}\label{subsubsection:DCL-ICL-CMF}
% Then, observe that the orthonormal basis $K_{t}$ \textit{grows} its columns adaptively, which makes $\ICL^\dagger$ into an \textit{\textbf{expansion-based method}}.

In \S \ref{subsubsection:ICL-CMF}, $\ICL^\dagger$ is shown to be an \textit{\textbf{expansion-based method}} that grows the columns of the orthonormal basis $K_{t}$ adaptively. Recall that the challenges of designing \textit{\textbf{expansion-based methods}} include (1) how many parameters to add and (2) where to add them. For the problem of  continual matrix factorization, $\ICL^\dagger$ addresses these two challenges perfectly by leveraging the eigen structures of features $\bY_t$ or projected features $Y_t$. To our knowledge, no \textit{\textbf{expansion-based methods}} in deep continual learning have exploited such structures; this implies an important extension as future work. Also, different from many existing \textit{\textbf{expansion-based methods}} (as \citet{Yan-CVPR2021} reviewed), $\ICL^\dagger$ does not require the identity of the task at the test time.

% Differently from Remark \ref{remark:projection-methods} where many \textit{\textbf{memory-based optimization methods}} are approximately the Continual Learner$^\dagger$, existing \textit{\textbf{expansion-based methods}} in  deep continual learning have not exploited such eigen structures, to the best of our knowledge.

The other important aspect that $\ICL^\dagger$ reveals is \textit{\textbf{the role of the network width in resisting catastrophic forgetting}}. Indeed, the matrix $K_t$ can be viewed as the second layer of a two-layer linear network, and the growth of its columns corresponds to increasing the network width. This principled way of increasing the network width complements the empirical claim of \citet{Mirzadeh-ICML2022} and the experimental observation of \citet{Rusu-arXiv2016,Yoon-ICLR2018}:  \textit{Wide neural networks forget less catastrophically}. Two remarks are in order, to finish the section:% Note that, in continual matrix factorization, increasing the width is necessary to prevent forgetting as the dimension of $\sum_{i=1}^{t}\cS_i$ increases.

% In retrospect, we note that $K_{t}$ \textit{grows} adaptively to  incorporate new knowledge and resists catastrophic forgetting. Since matrix factorization is ``equivalent'' to two-layer linear neural networks, our discussion here might shed theoretical light on the empirical effectiveness of \textit{expansion-based approaches} in deep continual learning, where the number of parameters (e.g., $\overline{U}_t$) are increased to account for the incoming task; cf. \citet{Ramesh-ICLR2022}. 

% We end the section with two remarks (in homage to Orwell):

\begin{remark}[Grow $=$ Shrink]\label{remark:grow}
	While $\bK_t$ \textit{shrinks} in continual linear regression and we have $\cK_t\subset \cK_{t-1}$ in general, we see that, in continual matrix factorization, $K_t$ \textit{grows} its columns as  $\ICL^\dagger$ is presented with more subspaces. However, the sum $\sum_{i=1}^{t}\cS_i$ can be uniquely identified with the intersection $\cap_{i=1}^t \cS_i^\perp$ of the orthogonal complement $\cS_i^\perp$, and therefore if we store a basis matrix for $\cap_{i=1}^t \cS_i^\perp$ (instead of for $\sum_{i=1}^{t}\cS_i$), then the memory consumption shrinks over time. We will elaborate on this idea in Appendix \ref{section:Dual-CPCA}.
\end{remark}
\begin{remark}[Remember $=$ Forget]\label{remark:remember}
    While $\cK_{t-1}$ \textit{remembers} all common global minimizers,  $\ICL^\dagger$ actually \textit{forgets} $\cK_{t-1}$ as it \textit{never updates} $\cK_{t-1}$. Put differently, the ability of a learner to improve the performance on previous tasks with knowledge from new tasks is desired (only) when previous tasks are solved sub-optimally; such ability is typically called \textit{positive backward transfer} \citep{Lin-NeurIPS2022}. % \citep{Prado-CoLLAs2022}. 
\end{remark}

\section{Continual Learning Basics: Generalization}\label{section:generalization}
In this section we prove that  $\ICL^\dagger$ is a learner in the \textit{statistical learning} sense (therefore we can remove $\dagger$); see, e.g., \citet{Shalev-book2014,Mohri-book2018} for some basics of statistical learning. Towards that goal, we now reframe our problem setup in a statistical learning environment. Assume that, for task $t$, the dataset $D_t=\{\bd_{ti} \}_{i=1}^{m_t}$ consists of $m_t$ \textit{independent and identically distributed} samples $\bd_{ti}\overset{\textnormal{i.i.d.}}{\sim} \cD_t$, where $\cD_t$ is some (unknown) data distribution for task $t$, and the objective $L_t$ \eqref{eq:task-t-continual-dagger} is given as
\begin{align}\label{eq:ERM}
	L_t(\bw; D_t) := \frac{1}{m_t}\sum_{i=1}^{m_t}\ell_t (\bw; \bd_{ti}),
\end{align}
where $\ell_t(\bw; \bd_{ti})$ is the loss on sample $\bd_{ti}$ evaluated at $\bw$. Hence, \eqref{eq:task-t-continual-dagger} becomes an \textit{empirical risk minimization} problem, and $\ICL^\dagger$ becomes a (constrained) empirical risk minimizer.

Corresponding to \eqref{eq:ERM} is the statistical learning task
\begin{align}\label{eq:population}
	\cG^*_t: = \argmin_{\bw\in \cW} \bbE_{\bd\sim \cD_t }[ \ell_t( \bw; \bd ) ].
\end{align}
To proceed, we need the assumption of \textit{shared multitask model} and the Ideal Continual Learner (without $\dagger$):
\begin{assumption}[Shared Multitask Model]\label{assumption:realizability}
	$\cap_{t=1}^T\cG^*_t\neq \varnothing$.
\end{assumption}
\begin{definition}[The Ideal Continual Learner, $\ICL$]\label{definition:CL}
	With $\cK^*_0:=\cW$, for  $t=1,2,\dots,T$, the algorithm $\ICL$ solves
	\begin{align}\label{eq:task-t-continual}
		\cK^*_t \gets\argmin_{\bw \in \cK^*_{t-1} } \bbE_{\bd\sim \cD_t }[ \ell_t(\bw; \bd) ].
	\end{align}
\end{definition}
Much of what we said to  $\ICL^\dagger$ and Assumption \ref{assumption:realizability-dagger} applies to the Ideal Continual Learner ($\ICL$) and Assumption \ref{assumption:realizability}; we shall not repeat  here. Instead, we will investigate the possibility of minimizing \eqref{eq:population} or \eqref{eq:task-t-continual} via $\ICL^\dagger$.

One difficulty of such investigation is as follows. While $\cG_t$ of \eqref{eq:task-t-dagger} approaches $\cG^*_t$ of \eqref{eq:population} as sample size $m_t$ tends to infinity, in general we have $\cG_t\neq \cG^*_t$ when $m_t$ is finite. As such, Assumptions \ref{assumption:realizability-dagger} and \ref{assumption:realizability} are different (though related), and, in general, one of them does not imply the other. As a consequence, the sufficiency of  $\ICL$ under Assumption \ref{assumption:realizability} does not imply the sufficiency of  $\ICL^\dagger$, which requires Assumption \ref{assumption:realizability-dagger} (cf. Proposition \ref{prop:sufficiency}); and vice versa. This difficulty disappears if we make both Assumptions \ref{assumption:realizability-dagger} and \ref{assumption:realizability}, which, however, would ask for too much.

We will proceed with either Assumption \ref{assumption:realizability-dagger} (\S \ref{subsection:multitask}) or Assumption \ref{assumption:realizability} (\S \ref{subsection:constrained-learner}), but not both. To do so, we need some assumptions on the objective functions $L_t$ and hypothesis space $\cW$, which are standard in statistical learning:
\begin{assumption}[Uniform Convergence]\label{assumption:uniform-convergence}
    Let $B$ and $M$ be two positive constants. Assume $\| \bw \|_2\leq B$ for every $\bw\in\cW$. Assume that $\ell_t(\bw; \bd)$ is $M$-Lipschitz in $\bw$ for every sample $\bd\sim \cD_t$ and every $t=1,\dots,T$, i.e.,
	\begin{align*}
		| \ell_t(\bw; \bd) - \ell_t(\bw'; \bd)  | \leq M \cdot \|\bw - \bw' \|_2,\ \forall \bw,\bw'\in\cW.
	\end{align*}
	These assumptions suffice to ensure \textit{uniform convergence} \citep[Thm 5]{Shalev-CoLT2009}:  For fixed $t$ and $\delta\in(0,1)$, with probability at least $1-\delta$ we have ($\forall \bw \in \cW$)
	\begin{align}\label{eq:uniform-convergence}
		&| L_t (\bw; D_t) -  \bbE_{\bd\sim \cD_t }[ \ell_t( \bw; \bd ) ] |\leq \zeta(m_t,\delta), \\
			\text{where\  } & \zeta(m_t,\delta):= O\bigg(\frac{ MB  \sqrt{n \log (m_t) \log(n/\delta)} }{\sqrt{m_t}}  \bigg) \nonumber
	\end{align}
	approaches zero as the sample size $m_t$ tends to infinity. 
\end{assumption}
The term $\zeta(m_t,\delta)$ defined in \eqref{eq:uniform-convergence} serves as a uniform convergence bound and will appear frequently in our generalization bounds. While $\zeta(m_t,\delta)$ approaches zero as the sample size $m_t$ tends to infinity (assuming other terms are fixed), it is also controlled by other factors as  \eqref{eq:uniform-convergence} suggests. First of all, it depends on the failure probability $\delta$, and we would like to set $\delta$ small. Second, it depends on the Lipschitz continuity constant $M$ of the loss function; in the context of deep learning, $M$ is controlled by network architectures and its value is typically unknown, although it might be estimated in certain cases via computationally intensive algorithms \citep{Fazlyab-NeurIPS2019}. Third, $\zeta(m_t,\delta)$ also depends on the dimension $n$ of the variable $\bw$ (e.g., the number of parameters in a deep network). In particular, we have $\zeta(m_t,\delta)=O(\sqrt{n\log n/m_t})$ (ignoring other parameters), and so, for $\zeta(m_t,\delta)$ to be small, we need $m_t\gg n\log n$. 

%Figure \ref{fig:generalization} visualizes the high-level logical relation of our results in \S \ref{subsection:multitask} and \S \ref{subsection:constrained-learner}.

\begin{figure*}
    \centering
    \tikzset{
  big arrow/.style={
    decoration={markings,mark=at position 1 with {\arrow[scale=1.5,#1]{>}}},
    postaction={decorate},
    shorten >=0.3pt},
  big arrow/.default=blue}
    \begin{tikzpicture}[node distance=2cm]

        % nodes
        \node [draw] at (0, 0) (A) {Assumption \ref{assumption:realizability-dagger}};
        \node [draw] at (0,-3) (C) {The Ideal Continual Learner$^\dagger$};

        \node [text width=2cm] at (4,-1) (A3) {Assumption \ref{assumption:uniform-convergence}};
        
        \node [draw] at (8,0) (B) {Assumption \ref{assumption:realizability}};
        \node [draw] at (8,-3) (D) {The Ideal Continual Learner};
        
        % arrows
        % \draw [->] (A) -- (B) -- (C) -- (D);
        \draw[-to,big arrow=black,thick] (A) -- (C);
        \draw[-to,big arrow=black,thick] (B) -- (D);

        \draw[-to,big arrow=black,thick,dashed] (A) -- (D);
        \draw[-to,big arrow=black,thick,dashed] (B) -- (C);
        
    \end{tikzpicture}
    \caption{Figure \ref{fig:generalization} is designed to help the reader understand our results of \S \ref{subsection:multitask} and \S \ref{subsection:constrained-learner} at a high level, and it can be read as follows. Assumptions \ref{assumption:realizability-dagger} and \ref{assumption:realizability} guarantee the sufficiency of the Ideal Continual Learner$^\dagger$ and the Ideal Continual Learner, respectively, and this is denoted by solid arrows in the figure. Assumptions \ref{assumption:realizability-dagger} and \ref{assumption:uniform-convergence} indicate an approximate sufficiency of the Ideal Continual Learner$^\dagger$ (cf. Theorem \ref{theorem:CL1} and Remark \ref{remark:approximate-sufficiency}), and this is denoted by a dashed arrow in the figure. }
    \label{fig:generalization}
\end{figure*}
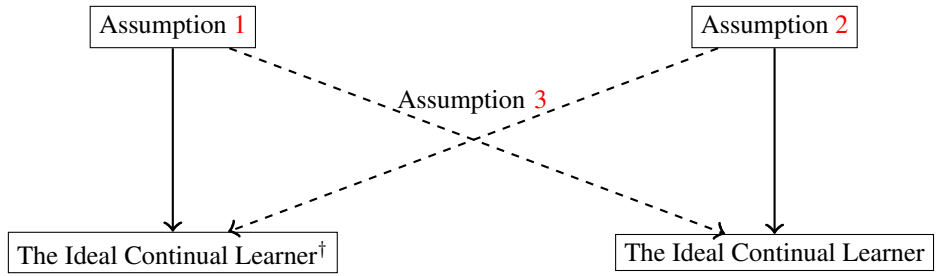

\subsection{\texttt{ICL}$^\dagger$ Learns All Tasks}\label{subsection:multitask}
In this section, we proceed with Assumption \ref{assumption:realizability-dagger} and without Assumption \ref{assumption:realizability}. Our main result is this (see also Figure \ref{fig:generalization}):
\begin{theorem}[$\ICL^\dagger$ $\Rightarrow$ All-Task Learner]\label{theorem:CL1}
	Let $c^*_t$ be the minimum value of \eqref{eq:population}.  Let $\hat{\bw}\in\cK_T$. Let $\delta\in(0,1)$ and $\zeta(m_t,\delta)$ be as in Assumption \ref{assumption:uniform-convergence}. Under Assumptions \ref{assumption:realizability-dagger} and \ref{assumption:uniform-convergence}, with probability at least $1-\delta$, we have ($\forall t=1,\dots,T$)
	\begin{align}\label{eq:generalization-bound-CL1}
		c_t^*\leq \bbE_{\bd\sim \cD_t }[ \ell_t( \hat{\bw}; \bd ) ]  \leq c^*_t + \zeta(m_t, \delta/T).
	\end{align}
\end{theorem}
As \eqref{eq:uniform-convergence} implies, an exponentially large $T$ would make the numerator of the error term $\zeta(m_t, \delta/T)$  in \eqref{eq:generalization-bound-CL1} large, while a large number of samples $m_t$ would make it small. In the extreme case $T\to\infty$, it is necessary to have $m_t\to \infty$, otherwise $\zeta(m_t, \delta/T)$ would be infinity and the upper bound of \eqref{eq:generalization-bound-CL1} would be invalid. In the current deep continual learning practice, though, $T$ is of the constant order, e.g., $T\leq 1000$ \citep{Lesort-arXiv2022}. Note then that, for fixed $T$ and $\delta$, as the sample size $m_t$ tends to $\infty$ for every $t=1,\dots,T$, we have $\zeta(m_t,\delta/T)\to 0$, in which case $\ICL^\dagger$ becomes an all-task learner: Every point $\hat{\bw}_T$ of $\cK_T$ that it finds reaches the minimum values of all learning tasks \eqref{eq:population}. 
\begin{remark}[Approximate Sufficiency]\label{remark:approximate-sufficiency}
    Assumptions \ref{assumption:realizability-dagger} and \ref{assumption:uniform-convergence} promise $\hat{\bw}$ \eqref{eq:generalization-bound-CL1} to be an \textit{approximate} global minimizer to all learning tasks \eqref{eq:population}; or we might say that they make $\ICL$ \textit{approximately sufficient} even without Assumption \ref{assumption:realizability}. 
\end{remark}

\subsection{Relaxed Continual Learner$^\dagger$ $=$ \texttt{ICL}}\label{subsection:constrained-learner}
Here, we consider Assumption \ref{assumption:realizability}, at the risk that $\ICL^\dagger$ might not be sufficient. However, by deriving generalization bounds, we will prove that a \textit{relaxation} of $\ICL^\dagger$ is $\ICL$.

While $\ICL^\dagger$ can perform arbitrarily worse for certain pathological cases in the absence of Assumption \ref{assumption:realizability-dagger} (as Example \ref{example:insufficiency} showed), Assumptions \ref{assumption:realizability} 
and \ref{assumption:uniform-convergence}  come into play, making $\ICL^\dagger$ \textit{approximately sufficient} (similarly to Remark \ref{remark:approximate-sufficiency}):
\begin{prop}[Approximate Sufficiency]\label{prop:approximately-optimal}
	Let $\delta\in(0,1)$. Assumptions \ref{assumption:realizability}, \ref{assumption:uniform-convergence} imply there is a point $\bw\in\cW$ satisfying
	\begin{align*}
		L_t(\bw; D_t) \leq c_t + \zeta(m_t,\delta/T), \ \ \forall t=1,\dots,T
	\end{align*}
	 with probability at least $1-\delta$. Here $c_t$ and $\zeta(m_t,\delta)$ are respectively defined in Proposition \ref{prop:inequality-constrained-CL-dagger} and Assumption \ref{assumption:uniform-convergence}.
\end{prop}
In the absence of Assumption \ref{assumption:realizability-dagger}, the  constraint of $\ICL^\dagger$ in Proposition \ref{prop:inequality-constrained-CL-dagger} might be prohibitive and lead to sub-optimal solutions (cf. Example \ref{example:insufficiency}). This is why we relax it as per Proposition \ref{prop:approximately-optimal} by some additive factor:
\begin{definition}[Relaxed Continual Learner$^\dagger$]\label{definition:relaxed-CL}
	With $\cK_0:=\cW$, the \textit{relaxed} Continual Learner$^\dagger$ is an algorithm  that solves the following program sequentially for  $t=1,2,\dots,T$:
	\begin{align}\label{eq:inequality-constrained-dagger-relaxed}
		&\ \min_{\bw \in\cW} L_t(\bw; D_t) \\
		\textnormal{s.t.} &\ L_i(\bw; D_i) \leq c_i+\zeta(m_i, \delta/t),\ \forall i=1,\dots, t-1 \nonumber
	\end{align}
\end{definition}
% We now show that the relaxed Continual Learner$^\dagger$ is $\ICL$:
We can now state the following (see also Figure \ref{fig:generalization}):
\begin{theorem}[Relaxed Continual Learner$^\dagger$ $\approx$ $\ICL$]\label{theorem:CL2}
	Let $\delta\in(0,1)$. Let $c^*_t$ be the minimum of \eqref{eq:population} and $\zeta(m_t,\delta)$ defined in \eqref{eq:uniform-convergence}. Suppose Assumptions \ref{assumption:realizability} and \ref{assumption:uniform-convergence} hold. With probability at least $1-\delta$, the relaxed Continual Learner$^\dagger$ is feasible, and its global minimizer $\overline{\bw}$ satisfies ($\forall t=1,\dots,T$)
	\begin{align}\label{eq:generalization-bound-CL2}
		c_t^*\leq \bbE_{\bd\sim \cD_t }[ \ell_t( \overline{\bw}; \bd ) ]  \leq c^*_t + \zeta(m_t,\delta/T).
	\end{align}
\end{theorem}

\subsection{\texttt{ICL} and Constrained Learning}\label{subsection:constrained-learning}
Similarly to Proposition \ref{prop:inequality-constrained-CL-dagger}, the constraint $\bw\in\cK^*_{t-1}$ of $\ICL$ can be written as inequality constraints $\bbE_{\bd\sim \cD_{i} }[ \ell_{i}( \bw; \bd ) ]  \leq c^*_{i}$ ($\forall i=1,\dots,t-1$), where $c^*_{i}$ is the minimum for task $i$ \eqref{eq:population}. If the values of $c^*_i$ are instead determined by applications (not as minima of prior tasks), then the resulting learning problem coincides with \textit{constrained learning} \citep{Chamon-TIT2022}. Constrained learning naturally arises when learning under certain requirements, e.g., robustness \citep{Zhang-ICML2019,Robey-NeurIPS2021}, safety \citep{Paternain-NeurIPS2019}, or fairness \citep{Cotter-JMLR2019}. In a different line of research, constrained learning is also called \textit{stochastic optimization with expectation constraints} \citep{Yu-NeurIPS2017,Bedi-TSP2019,Madavan-JOTA2021,Akhtar-TSP2021}. For both problems, several algorithms have recently emerged with theoretical guarantees. For example, \citet{Chamon-TIT2022} solve the constrained learning problem using Lagrangian multipliers, which can be viewed as a \textit{\textbf{regularization-based method}} as per a common deep continual learning taxonomy. Left though to future work, leveraging this line of research might be key to improving the current continual learning systems. Finally, note that \citet{Chamon-TIT2022} developed generalization bounds for constrained learning via different mathematical mechanisms (e.g., primal-dual analysis), and we refer the reader to their works for a detailed exposition of these ideas.

% The constrained learning formulation \eqref{eq:constrained-learning} was discussed by \citet{Chamon-NeurIPS2020,Chamon-TIT2022}, and it naturally arises when learning under certain requirements, e.g., robustness \citep{Zhang-ICML2019,Robey-NeurIPS2021}, safety \citep{Paternain-NeurIPS2019}, or fairness \citep{Cotter-JMLR2019}. In a different line of research, the constrained learning problem \eqref{eq:constrained-learning} is also called \textit{stochastic optimization with expectation constraints} \citep{Yu-NeurIPS2017,Bedi-TSP2019,Madavan-JOTA2021,Akhtar-TSP2021}; see Table 1 of \citet{Akhtar-TSP2021} for a summary of existing algorithms.

\subsection{\texttt{ICL} and Rehearsal}\label{subsection:rehearsal} 
Using the $\ICL$ framework, we can now acquire a theoretical understanding of research and debate \ref{D2}. For each task $t=1,\dots,T-1$, assume we stored $s_t$ samples (out of $m_t$), and we now face task $T$. We \textit{rehearse}, i.e., retrain with all stored samples and data of task $T$, to minimize
\begin{align}\label{eq:rehearsal}
    \min_{\bw \in \cW} \frac{1}{m_T}\sum_{i=1}^{m_T}\ell_T (\bw; \bd_{Ti}) + \sum_{t=1}^{T-1}\frac{1}{s_t}\sum_{i=1}^{s_t}\ell_t (\bw; \bd_{ti}).
\end{align}
Rehearsal has shown good empirical performance in deep continual learning; see, e.g., \citet{Chaudhry-arXiv2019v4,Prabhu-ECCV2020,Zhang-NeurIPS2022a,Bonicelli-NeurIPS2022}. For the first time to our knowledge, a theoretical justification that accounts for its effectiveness is provided here:
\begin{theorem}[Rehearsal $\approx$ $\ICL$]\label{theorem:rehearsal}
    Under Assumptions \ref{assumption:realizability} and \ref{assumption:uniform-convergence}, with probability at least $1-\delta$, every global minimizer $\overline{\bw}$ of the rehearsal objective \eqref{eq:rehearsal} satisfies 
    \begin{equation}\label{eq:generalization-bound-CL3}
    \begin{split}
        \sum_{t=1}^{T} c^*_t \leq \sum_{t=1}^T \bbE_{\bd\sim \cD_t }[ \ell_t( \overline{\bw}; \bd ) ]  \leq  \sum_{t=1}^{T} c^*_t +  \textnormal{ERROR} \\ 
\textnormal{where\ }  \textnormal{ERROR}:= \zeta(m_T,\delta/T)  + \sum_{t=1}^{T-1} \zeta(s_t,\delta/T). 
    \end{split}
    \end{equation}
\end{theorem}

Note how \eqref{eq:generalization-bound-CL3} reveals the way rehearsal affects generalization: A larger memory buffer ($s_t$) means a smaller $\zeta(s_t,\delta/T)$, implying better generalization. Crucially, this bound remains the same as long as we store the same number of samples for each task (say $s_t$) \textit{\textbf{regardless of which samples we store}} (a natural consequence of the i.i.d. assumption). This deviates from the trendy idea of selecting \textit{representative} samples in recent methods---reportedly, these methods do not necessarily outperform a simple random selection mechanism \citep{Araujo-AACL2022}. Due to space, more elaborations are put in % Appendices \ref{subsection:elaboration-generalization} and 
Appendix \ref{subsection:elaboration-memory-selection}.

We can then compare the relaxed Continual Learner$^\dagger$ (\S \ref{subsection:constrained-learner}) and rehearsal (\S \ref{subsection:rehearsal}). Note first that one could also implement the relaxed Continual Learner$^\dagger$ using a subset of samples similar to rehearsal. Theoretically, the relaxed Continual Learner$^\dagger$ enjoys stronger generalization guarantees than rehearsal: Theorem \ref{theorem:CL2} bounds generalization errors for individual tasks, while Theorem \ref{theorem:rehearsal} only gives a bound on the multitask error. Computationally, rehearsal is simpler to implement than the relaxed Continual Learner$^\dagger$ and can maintain good performance \citep{Chaudhry-arXiv2019v4,Prabhu-ECCV2020,Zhang-NeurIPS2022a,Bonicelli-NeurIPS2022}.

\section{Conclusion}\label{section:conclusion}
Under the $\ICL$ framework, we derived and justified many existing methods in deep continual learning. Put conversely, \textit{all roads lead to Rome}: Many prior continual learning methods developed with different insights and different motivations turn out to be (approximately) $\ICL$. Importantly, we made several connections to other research fields. This activates opportunities for improving the existing (deep) continual learning systems, for which we venture to hope that $\ICL$ serves as a primary design principle.

The $\ICL$ framework in its present form comes with limitations. On the theoretical side, we tacitly assumed that the samples of each task $t$ are drawn i.i.d. from distribution $\cD_t$ in our generalization theory. While this i.i.d. assumption is standard in classic statistical learning, dispensing with it and accounting for out-of-distribution data within tasks is left as future work. On the algorithmic front, we emphasize that, in general, and particularly for deep continual learning, it is by no means easy to develop an \textit{exact} implementation of $\ICL$, which constitutes a major limitation of the proposed framework. Nevertheless, our paper has suggested many possibilities for approximating the Ideal Continual Learner. We believe devising such approximations with problem-specific insights will be a promising research direction.

% We tacitly assumed that the samples of each task $t$ are drawn i.i.d. from distribution $\cD_t$ in our generalization theory. While this i.i.d. assumption is standard in classic statistical learning, dispensing with it and accounting for out-of-distribution data within tasks is left as future work.

%\subsubsection*{Author Contributions}
%If you'd like to, you may include  a section for author contributions as is done
%in many journals. This is optional and at the discretion of the authors.
%
%\subsubsection*{Acknowledgments}
%Use unnumbered third level headings for the acknowledgments. All
%acknowledgments, including those to funding agencies, go at the end of the paper.

% \section*{Acknowledgements}
\myparagraph{Acknowledgements} This work is supported by the project ULEARN “Unsupervised Lifelong Learning” and co-funded under the grant number 316080 of the Research Council of Norway.

%If a paper is accepted, the final camera-ready version can (and
%probably should) include acknowledgements. In this case, please
%place such acknowledgements in an unnumbered section at the
%end of the paper. Typically, this will include thanks to reviewers
%who gave useful comments, to colleagues who contributed to the ideas,
%and to funding agencies and corporate sponsors that provided financial
%support.


\bibliography{Liangzu}
\bibliographystyle{icml2023}

%%%%%%%%%%%%%%%%%%%%%%%%%%%%%%%%%%%%%%%%%%%%%%%%%%%%%%%%%%%%%%%%%%%%%%%%%%%%%%%
%%%%%%%%%%%%%%%%%%%%%%%%%%%%%%%%%%%%%%%%%%%%%%%%%%%%%%%%%%%%%%%%%%%%%%%%%%%%%%%
% APPENDIX
%%%%%%%%%%%%%%%%%%%%%%%%%%%%%%%%%%%%%%%%%%%%%%%%%%%%%%%%%%%%%%%%%%%%%%%%%%%%%%%
%%%%%%%%%%%%%%%%%%%%%%%%%%%%%%%%%%%%%%%%%%%%%%%%%%%%%%%%%%%%%%%%%%%%%%%%%%%%%%%
\newpage
\appendix
\onecolumn

% \setlength{\epigraphwidth}{.64\textwidth}
% \epigraph{ \textit{My dear friend I promise you I will improve; I will no longer, as has ever been my habit, continue to ruminate on every petty vexation which fortune may dispense; I will enjoy the present, and the past shall be for me the past. [Translated by R.D. Boylan] } }{J. W. von Goethe, ``The Sorrows of Young Werther'' (1774)}

\section{Structure of The Appendix}
We structure the appendix as follows:
\begin{itemize}
    \item In Appendix \ref{section:elaboration}, we elaborate on several points for the reader to better understand and appreciate the main paper. 
    \item In Appendix \ref{section:OGD-proof}, we prove Proposition \ref{proposition:OGD}, showing that OGD is approximately the Ideal Continual Learner$^\dagger$ in the case of continual linear regression.
    \item In Appendix \ref{section:CL-review}, we review related works on continual learning, emphasizing deep continual learning methods that we mentioned in Remark \ref{remark:projection-methods}, existing assumptions on task relationships, and existing theoretical papers.
    \item We review related works in other related fields that are visualized in Figure \ref{fig:connection}. The emphasis is put on their connections to the Ideal Continual Learner and on detailed comparisons of our approach to related works. 
    \begin{itemize}
        \item In Appendix \ref{section:set-theoretical-est}, we review related works in set-theoretical estimation. This allows us to understand the Ideal Continual Learner from different perspectives.
        % \item In Appendix \ref{section:constrained-learning}, we review related works in constrained learning, and we make a detailed comparison to the work of \citet{Chamon-TIT2022}.
        \item In Appendix \ref{section:streaming-PCA}, we review streaming PCA, incremental SVD, and subspace tracking. An important remark there is that \textit{\textbf{expansion-based methods}} can be traced back to \citet{Bunch-1978} who proposed a method for incremental SVD.
    \end{itemize}
    \item In Appendix \ref{section:Dual-CPCA}, as promised in Remark \ref{remark:grow}, we present a \textit{dual} approach to continual matrix factorization, where the storage consumption shrinks over time. 
\end{itemize}

\section{Elaboration on The Main Paper}\label{section:elaboration}
Here we take the opportunity to elaborate on several points omitted in the main paper due to the lack of space.
\subsection{Elaboration on Terminologies}\label{subsection:terminology}
\begin{itemize}
    \item At first glance the Ideal Continual Learner seems to implicitly assume that the task identities are given during training, and therefore it can not be applied to the more general and more challenging case of \textit{task-agnostic} or \textit{task-free} continual learning, where in the training phase the learner has no access to task identities \citep{Zeno-arXiv2018v3,Zeno-NC2021,Aljundi-CVPR2019,Lee-ICLR2020,Jin-NeurIPS2021,Wang-ICML2022,Pourcel-ECCV2022,Ye-NeurIPS2022}. However, for task-agnostic or task-free continual learning, we receive a batch of samples each time, and we can simply regard it as a task, and compute the set of common global minimizers for this batch of data. Then we can formulate \eqref{eq:task-t-continual-dagger} for the next task (i.e., the next batch). Summarized informally, $\ICL$ is task-agnostic. However, a major limitation of this task-agnostic formulation of $\ICL$ is that it leads to a very large $T$, which compromises the generalization bounds derived in the main paper.  
    
    \item Methods of Remark \ref{remark:projection-methods} project the gradients onto certain subspaces, so they can be thought of as regularizing the gradients. This is why these methods are called \textit{\textbf{regularization-based methods}} in the taxonomy of \citet{Qu-arXiv2021}. In the paper, we call them \textit{\textbf{memory-based optimization methods}}, as they need to store some projection matrices and we believe it is important to highlight this extra memory consumption for continual learning. We  used the phrase \textit{\textbf{expansion-based methods}}, while \textit{\textbf{architecture-based methods}} and \textit{\textbf{structure-based methods}} are popular alternatives.
    \item While our formulation of the continual learning problem is general (e.g., each task can have different losses), our examples of continual linear regression and continual matrix factorization (\S \ref{section:CLR-CMF}) are specific in the sense that each task has the same loss and it is just that data samples are different for every task. This specific setting is called \textit{domain-incremental learning} in the taxonomy of \citet{Van-NMI2022}. Quote \citet{Van-NMI2022}:
\begin{center}
    \textit{Using task-specific components in this scenario is, however,
only possible if an algorithm first identifies the task, but that is
not necessarily the most efficient strategy. Preventing forgetting `by
design' is therefore not possible with domain-incremental learning, and
alleviating catastrophic forgetting is still an important unsolved challenge.}
\end{center}
    That said, we proved that continual learning without forgetting is possible under Assumption \ref{assumption:realizability-dagger}.
\end{itemize}

\subsection{Elaboration on Memory Selection Methods}\label{subsection:elaboration-memory-selection}
Here we review related works on memory selection methods for rehearsal in light of Theorem \ref{theorem:rehearsal}.

Arguably, rehearsal is a very simple and effective idea that balances memory consumption and the amount of forgetting. It has been commonly believed that, quoting \citet{Chaudhry-arXiv2019v4}, ``\textit{... the sample that the learner selects to populate the memory becomes crucial}''. Based on prior works \citep{Vitter-1985,Lopez-NeurIPS2017,Rebuffi-CVPR2017,Riemer-ICLR2019}, \citet{Chaudhry-arXiv2019v4} summarized four (seven resp.) basic methods for selecting samples. More recently, based on \citet{Isele-AAAI2018,Hayes-CVPR2021}, \citet{Araujo-AACL2022} further included three more basic methods, and compared the seven methods for natural language processing applications. The setting of \citet{Araujo-AACL2022} is domain-incremental learning, and our generalization bound \eqref{eq:generalization-bound-CL3} applies to the rehearsal mechanism verbatim, under the same assumptions. (Our result for the constrained optimization formulation is not applicable as it requires task identities.)

Table 1 of \citet{Araujo-AACL2022} presents the performance of these seven methods for text classification and question answering, while Table 2 presents running times. From the two tables, one can observe the performance difference between \textbf{N. Random} and \textbf{Reservoir} is statistically insignificant. Figure 1 of  \citet{Araujo-AACL2022} further shows that the two sampling methods, \textbf{N. Random} and \textbf{Reservoir}, tend to keep nearly the same number of samples for each task, while the samples that \textbf{N. Random} selects and those that \textbf{Reservoir} selects are in general very different. This corresponds well to Theorem \ref{theorem:rehearsal}, which formalizes an intuitive fact that selecting which samples to store does not matter under if the samples within each task fulfill the i.i.d. assumption. In particular, our generalization bound \eqref{eq:generalization-bound-CL3} is sensitive mainly to the number of samples for every task. 

It is thus important to ruminate on whether selecting the so-called \textit{representative} or \textit{diverse} samples is relevant for improving continual learning systems \citep{Araujo-AACL2022}. We believe that selecting representative samples from each task or in a streaming setting can still be beneficial \citep{Borsos-NeurIPS2020,Sun-ICLR2022} if the datasets contain some out-of-distribution data. This prompts incorporating an out-of-distribution detection module into continual learning frameworks.

% On the other hand, we omitted the details of computing $\bK_t$ from an SVD of $\overline{\bX}_t:=\bX_t \bK_{t-1}$ when discussing the implementation of $\ICL^\dagger$ for continual linear regression. Observe that, since $\bK_{t-1}$ is of full column rank, for any $\ba$, $\ba$ lies the nullspace of $\bX_t \bK_{t-1}$ if and only if $\bK_{t-1}\ba$ lies in $\text{null}(\bX_t) \cap \text{range} (\bK_{t-1})$. It therefore suffices to compute an orthonormal basis for the nullspace of $\bX_t \bK_{t-1}$, and then left-multiply this basis by $\bK_{t-1}$.

\section{Orthogonal Gradient Descent and Proposition \ref{proposition:OGD}}\label{section:OGD-proof}
\subsection{Orthogonal Gradient Descent for Deep Continual Learning}
We first review the orthogonal gradient descent algorithm (OGD) of \citet{Farajtabar-AISTATS2020} for deep continual learning in image classification. In this image classification application, the dataset $D_t$ for task $t$ is $\{ \bx_{ti}, \by_{ti} \}_{i=1}^{m_t}$, where $\bx_{ti}$ is some input image and $\by_{ti}$ is an one-hot vector of class labels. The training objective for task $t$ in this context is typically written as
\begin{align*}
	\min_{\bw\in\bbR^n} \sum_{i=1}^{m_t} \ell_\textnormal{CE} (  f(\bw; \bx_{ti}); \by_{ti} ),
\end{align*}
where $\ell_{\textnormal{CE}}(\cdot, \cdot)$ is the softmax cross entropy loss, and $f$ represents a deep network parameterized by $\bw\in\bbR^n$. The key idea of OGD is that, after training on task $1$ via gradient descent and obtaining a good (if not optimal) estimate $\widetilde{\bw}_1$, the training on task $2$ is still via gradient descent, but the gradient is replaced by its projection onto the orthogonal complement of the range space of $\nabla f_1:=[ \nabla f(\widetilde{\bw}_1; \bx_{11}),\dots, \nabla f(\widetilde{\bw}_1; \bx_{1m_1}) ]$. Here, the gradient $\nabla$ is evaluated with respect to the first parameter of $f$ (e.g., network weights). This idea is easily extended to multiple tasks. In particular, suppose we are now going to solve task $t$. For $i=1,\dots, t-1$, let $\widetilde{\bw}_i$ be the network weights obtained via training for tasks $1,\dots,i$ sequentially. For training on task $t$, The OGD method performs gradient descent, with the gradient  replaced by its projection onto the orthogonal complement of the range space of
\begin{align*} %\label{eq:OGD-grad}
	\nabla F_{t-1}:=[ \nabla f_1,\dots, \nabla f_{t-1}], \ \ \ \ \ \  \nabla f_{i}:=[ \nabla f(\widetilde{\bw}_i; \bx_{i1}),\dots, \nabla f(\widetilde{\bw}_i; \bx_{im_i}) ]\ \  (\forall i=1,\dots,t-1).
\end{align*}
In deep learning, the numbers of parameters and data samples are very large, so storing $\nabla F_{t-1}$ is prohibitive. This is improved by several recent follow-up works, which we mentioned in Remark \ref{remark:projection-methods} and will review in Appendix \ref{subsection:CL-review-projection}. Note that we will derive a theoretical analysis of OGD for continual linear regression (Appendix \ref{subsection:OGD-theory-CLR}), and our analysis is also applicable to methods of Remark \ref{remark:projection-methods} under similar assumptions.

\subsection{Orthogonal Gradient Descent for Continual Linear Regression}
We now derive the OGD algorithm for continual linear regression. For each task $t$, the model is $f(\bw; \bX_t ):= \bX_t\bw$ and the loss is the MSE loss. So we have $\nabla f_i = \bX_i^\top$ for every $i=1,\dots,t-1$ (i.e., the gradient is constant). Therefore, 
\begin{align*}
	\nabla F_{t-1}:=[ \nabla f_1,\dots, \nabla f_{t-1}] = [\bX_1^\top,\dots,\bX_{t-1}^\top].
\end{align*}
Denote by $\text{range}(\cdot)$ and $\text{null}(\cdot)$ the range space and nullspace of some matrix, respectively. By basic linear algebra we have
\begin{equation}\label{eq:range-null}
	\begin{split}
		\big( \text{range}(\nabla F_{t-1}) \big)^\perp &= \Big( \text{range}\big( [\bX_1^\top,\dots,\bX_{t-1}^\top] \big) \Big)^\perp \\
		&= \text{null}\big( [\bX_1^\top,\dots,\bX_{t-1}^\top]^\top \big) \\
		&=\text{null} (\bX_1) \cap \cdots \cap \text{null} (\bX_{t-1}) \\
		&= \text{range}(\bK_{t-1}).
	\end{split}
\end{equation}
Here we recall that we defined $\bK_{t-1}$ as an orthonormal basis matrix of $\text{null} (\bX_1) \cap \cdots \cap \text{null} (\bX_{t-1})$ in \S \ref{subsection:CLR}. Since $\bK_{t-1}^\top \bK_{t-1}$ is the identity matrix, the matrix representation of this projection is $\bK_{t-1} \bK_{t-1}^\top$. When solving task $t$, OGD projects the gradient of the loss $\|\bX_t \bw - \by_t \|_2^2$  at each iteration onto $ \text{range}(\bK_{t-1})$, and then performs a descent step with this modified gradient. We next describe the OGD algorithm more formally.

Let $k$ be an iteration counter, $\gamma^{(k)}_t$ the stepsize for task $t$ at iteration $k$. Set the initialization $\bw_t^{(0)}$ for task $t$ as the final weight $\widetilde{\bw}_{t-1}$ after training the previous tasks sequentially, i.e., $\bw_t^{(0)}\gets\widetilde{\bw}_{t-1}$, while for the first task we use an arbitrary initialization $\bw_1^{(0)}$. For task $1$, OGD performs gradient descent
\begin{align*}%\label{eq:OGD-task1}
    \bw_{1}^{(k+1)} \gets \bw_{1}^{(k)} - \gamma^{(k)}_1 \bh_t^{(k)}, \ \ \  \bh_1^{(k)}:= 2\bX_1^\top ( \bX_1 \bw_{1}^{(k)} - \by_1 ).
\end{align*}

For task $t>1$, its update rule is defined as (recall \eqref{eq:OGD-CLR})
\begin{align}\label{eq:OGD-taskt}
	\bw_{t}^{(k+1)} \gets \bw_{t}^{(k)} - \gamma^{(k)}_t \cdot \bK_{t-1} \bK_{t-1}^\top \bh_t^{(k)}, \ \ \  \bh_t^{(k)}:= 2\bX_t^\top ( \bX_t \bw_{t}^{(k)} - \by_t ).
\end{align}
\subsection{Theoretical Analysis of Orthogonal Gradient Descent for Continual Linear Regression}\label{subsection:OGD-theory-CLR}
In this section we interpret what we meant by Proposition \ref{proposition:OGD}. Let us start with a simplified situation where OGD is assumed to already find a common global minimizer of the first $t-1$ tasks.
\begin{prop}\label{prop:OGD-ideal}
    Let $t>1$. Suppose Assumption \ref{assumption:realizability-dagger} holds for continual linear regression. Assume that the weight $\widetilde{\bw}_{t-1}$ produced by OGD after training the first $t-1$ tasks is a common global minimizer of tasks $1,\dots,t-1$. Then the update formula \eqref{eq:OGD-taskt} of OGD for task $t$ is equivalent to a gradient descent step applied to the objective \eqref{eq:continual-CLR-unconstrained} of $\ICL^\dagger$.
\end{prop}
\begin{proof}
    Since OGD is initialized at a common global minimizer $\bw_{t}^{(0)}=\widetilde{\bw}_{t-1}$, we can write $\bw_{t}^{(0)}=\widetilde{\bw}_{t-1} + \bK_{t-1} \ba_{t}^{(0)}$ for a unique $\ba_{t}^{(0)}$ (actually $\ba_{t}^{(0)}=0$). Define $\overline{\bX}_t:=\bX_t\bK_{t-1}$ and
    \begin{align}\label{eq:tmp-GD-CL}
        \ba_{t}^{(1)} := \ba_{t}^{(0)} - \gamma^{(0)}_1 \bg_1^{(0)}, \ \ \ \ \bg_1^{(0)}:= 2\overline{\bX}_t^\top (\overline{\bX}_t \ba_{t}^{(0)} +  \bX_t \widetilde{\bw}_{t-1} - \by_t ).
    \end{align}
    Note that \eqref{eq:tmp-GD-CL} is exactly a gradient descent step applied to \eqref{eq:continual-CLR-unconstrained} (with $\hat{\bw}_{t-1}$ replaced by a ``different'' particular solution $\widetilde{\bw}_{t-1}$ of first $t-1$ tasks). As per \eqref{eq:OGD-taskt}, the first iteration of OGD is 
    \begin{equation}\label{eq:tmp-OGD=GD}
       \begin{split}
           &\ \bw_{t}^{(1)} \gets \bw_{t}^{(0)} - \gamma^{(0)}_t \cdot \bK_{t-1} \bK_{t-1}^\top \big( 2\bX_t^\top ( \bX_t \bw_{t}^{(0)} - \by_t ) \big) \\
       \Leftrightarrow&\ \bw_{t}^{(1)} = \widetilde{\bw}_{t-1} + \bK_{t-1} \ba_{t}^{(0)} - \gamma^{(0)}_t \cdot \bK_{t-1}  \Big( 2\overline{\bX}_t^\top \big( \bX_t (\widetilde{\bw}_{t-1} + \bK_{t-1} \ba_{t}^{(0)})- \by_t \big) \Big) \\
       \Leftrightarrow&\ \bw_{t}^{(1)} = \widetilde{\bw}_{t-1} + \bK_{t-1}(\ba_{t}^{(0)} - \gamma^{(0)}_t \bg_1^{(0)}) \\
       \Leftrightarrow&\ \bw_{t}^{(1)} = \widetilde{\bw}_{t-1} + \bK_{t-1} \ba_{t}^{(1)}
       \end{split}
    \end{equation}
    Now we see from  \eqref{eq:tmp-GD-CL} and \eqref{eq:tmp-OGD=GD} that the first iteration of OGD can be thought of first computing the coefficient vector $\ba_t^{(1)}$ via gradient descent \eqref{eq:tmp-GD-CL} and then updating $\bw_{t}^{(1)}$ via $\bw_{t}^{(1)} = \widetilde{\bw}_{t-1} + \bK_{t-1} \ba_{t}^{(1)}$. Since $\widetilde{\bw}_{t-1}$ and $\bK_{t-1}$ are fixed, the update of $\bw_t^{(1)}$ in the first iteration of OGD can actually be viewed as calculating $\ba_{t}^{(1)}$ implicitly. Finally, one can easily verify that the above derivation applies to any iteration $k$ (e.g., by changing the indices from $0$ to $k$ and from $1$ to $k+1$).
\end{proof}
With a suitable choice of stepsizes, gradient descent converges to an optimal solution to \eqref{eq:continual-CLR-unconstrained} at the number of iterations $k$ approaches infinity; see, e.g., Theorem 2.1.14 of \citet{Nesterov-2018}. As a consequence, under the assumption of Proposition \ref{prop:OGD-ideal}, the OGD algorithm applied to continual linear regression converges to an optimal solution of \eqref{eq:continual-CLR} as $k\to\infty$.

The above analysis relies on simplified assumptions that OGD can find an exact solution to the first $t-1$ tasks, and that one can run OGD for infinitely many iterations. Dispensing with these assumptions to reach a similar conclusion is not hard: 
\begin{prop}\label{proposition:OGD-formal}
    Assume $t>1$. Suppose Assumption \ref{assumption:realizability-dagger} holds for continual linear regression. Let $\widetilde{\bw}_{t-1}$ be the weight produced by the OGD algorithm after training the first $t-1$ tasks. Then the update formula \eqref{eq:OGD-taskt} of the OGD algorithm for task $t$ is equivalent to a gradient descent step applied to the following problem
    \begin{equation}\label{eq:continual-CLR-OGD}
   \min_{\ba} \| \bX_t (\widetilde{\bw}_{t-1} + \bK_{t-1} \ba) - \by_t \|_2^2.
\end{equation}
\end{prop}
\begin{proof}
    The proof follows directly from that of Proposition \ref{prop:OGD-ideal}. 
\end{proof}
The only difference of \eqref{eq:continual-CLR-OGD} from the objective \eqref{eq:continual-CLR-unconstrained} of $\ICL^\dagger$ is that $\widetilde{\bw}_{t-1}$ now might not be optimal for all previous tasks. Therefore, OGD is approximately equivalent to $\ICL^\dagger$, and its approximation quality depends on how close $\widetilde{\bw}_{t-1}$ is to the set $\cK_{t-1}$ of common solutions to previous $t-1$ tasks.

\section{Continual Learning: Related Works}\label{section:CL-review}
There are now many existing methods for (deep) continual learning, including \textit{\textbf{regularization-based methods}} \citep{Kirkpatrick-2017,Zenke-ICML2017,Liu-ICPR2018,Ritter-NeurIPS2018,Lesort-arXiv2019,Park-ICCV2019,Yin-arXiv2020v3,Zhou-CVPR2021,Heckel-AISTATS2022}, \textit{\textbf{memory-based methods}} \citep{Robins-1993,Shin-NeurIPS2017,Rebuffi-CVPR2017,Lopez-NeurIPS2017,Chaudhry-arXiv2019v4,Aljundi-NeurIPS2019,Buzzega-NeurIPS2020,Verwimp-ICCV2021,Bang-CVPR2021,Kim-CoLLAs2022,Zhang-NeurIPS2022a,Wang-ICLR2022,Saha-WACV2023},  \textit{\textbf{expansion-based methods}} \citep{Rusu-arXiv2016,Yoon-ICLR2018,Mallya-CVPR2018,Li-ICML2019b,Hung-NeurIPS2019,Ramesh-ICLR2022,Douillard-CVPR2022}, or combinations thereof \citep{Yan-CVPR2021,Xie-CVPR2022,Wang-CVPR2022,Frascaroli-arXiv2023}; for surveys, see, e.g., \citet{Parisi-NN2019,Delange-TPAMI2021,Qu-arXiv2021,Shaheen-JIRS2022,Van-NMI2022}. 

Our review here has two purposes. First, we will discuss the methods of Remark \ref{remark:projection-methods} in detail as they are highly related to $\ICL^\dagger$. Second, we will discuss related theoretical works through the lens of task relationships.

\subsection{Memory-Based Optimization Methods (Projection-Based Methods)}\label{subsection:CL-review-projection}
\myparagraph{The Basic Idea} Suppose now we have a three-layer deep network $f$ in the following form:
\begin{align}\label{eq:3layer}
    f\big((\bW_1,\bW_2,\bW_3); \bX^{(0)}_t\big) := \sigma_3 (\sigma_2( \sigma_1(\bX^{(0)}_t \bW_1) \bW_2 )  \bW_3 )
\end{align}
Here, each $\sigma_i$ is some non-linearity activation function, $\bW_i$ matrices of learnable weights, and $\bX^{(0)}_t$ an input data matrix at the $t$-th batch (or for the $t$-th task). Every column of $\bX^{(0)}_t$ is a data sample, and the number of columns corresponds to the batch size. We consider three layers as an example; what follows applies to multiple layers and is without loss of generality.

Let us write the network $f$ \eqref{eq:3layer} equivalently as
\begin{align}
     f\big((\bW_1,\bW_2,\bW_3); \bX^{(0)}_t\big) =  \sigma_3( \bX^{(2)}_t \bW_3),\ \ \ \bX^{(2)}_t := \sigma_2( \bX^{(1)}_t \bW_2), \ \ \ \bX^{(1)}_t := \sigma_1( \bX^{(0)}_t \bW_1).
\end{align}
To proceed, suppose for every $j=0,1,2$ the intersection $\cap_{i=1}^{t-1}\text{null}(\bX_i^{(j)})$ of nullspaces is not empty, and denote by $\bK^{(j)}_{t-1}$ an orthonormal basis matrix of $\cap_{i=1}^{t-1}\text{null}(\bX_i^{(j)})$. This matrix $\bK^{(j)}_{t-1}$ plays the same role as the matrix $\bK_{t-1}$ that we defined for continual linear regression (\S \ref{subsection:CLR}), and the difference is that we now defined such matrices for every linear layer. 

Suppose now we have the data $\bX^{(0)}_t$ for task (batch) $t$ available. Let $\bG_1,\bG_2$ and $\bG_3$ be the gradients calculated via data $\bX^{(0)}_t$ and they account respectively for weights $\bW_1,\bW_2$, and $\bW_3$. Let $\gamma_1,\gamma_2,\gamma_3$ be stepsizes. Then the weight update
\begin{equation}\label{eq:update}
    \begin{split}
        \bW_1^+ \gets \bW_1 - \gamma_1 \bK^{(1)}_{t-1} (\bK^{(1)}_{t-1})^\top \bG_1 \\ \bW_2^+ \gets \bW_2 - \gamma_2 \bK^{(2)}_{t-1} (\bK^{(2)}_{t-1})^\top \bG_2 \\ \bW_3^+ \gets \bW_3 - \gamma_2 \bK^{(3)}_{t-1} (\bK^{(3)}_{t-1})^\top \bG_3
    \end{split}
\end{equation}
never changes the output of $f$ on previous batches; for every $i<t$, it holds that 
\begin{align*}
    &\ \bX_i^{(0)}\bW_1^+ = \bX_i^{(0)}\bW_1, \ \ \, \bX_i^{(1)}\bW_2^+ = \bX_i^{(1)}\bW_2, \ \ \ \bX_i^{(2)}\bW_3^+ = \bX_i^{(2)}\bW_3 \\
    \Rightarrow&\ f\big((\bW_1^+,\bW_2^+,\bW_3^+); \bX^{(0)}_i\big) = f\big((\bW_1,\bW_2,\bW_3); \bX^{(0)}_i\big).
\end{align*}
Yet, \eqref{eq:update} can decrease the loss of the current data $\bX^{(0)}_t$. This makes storing $\bK^{(j)}_{t-1}$ an appealing approach for preventing forgetting. In reality, each $\bX_i^{(j)}$ might not have a nullspace (for a small enough batch size), and the intersection $\cap_{i=1}^{t-1}\text{null}(\bX_i^{(j)})$ might simply be empty. A simple and effective rescue is to perform certain low-rank approximations; we will not go into the details here, and we omit some practically important implementation details.

\myparagraph{The Literature} The literature of continual learning has followed a slightly longer and perhaps more vivid trajectory to eventually identify the role of the projection matrix $\bK^{(j)}_{t-1} (\bK^{(j)}_{t-1})^\top$ in resisting forgetting. We review this trajectory next.

Inspired by classic works on adaptive filtering  \citep{Haykin-2002} and recursive least-squares \citep{Singhal-ICASSP1989} in signal processing, \citet{Zeng-NMI2019} proposed to set the projection matrix to be $I- (\bX_{1}^{(j)})^\top \big(  \bX_{1}^{(j)} (\bX_{1}^{(j)})^\top  +\lambda I \big)^{-1} \bX_{1}^{(j)}  $ for the task 1 and update it for later tasks as in recursive least-squares. Here $\lambda>0$ is some hyper-parameter and $I$ identity matrix, and they are used to prevent the case where $\bX_{1}^{(j)} (\bX_{1}^{(j)})^\top$ is not invertible. The method of \citet{Zeng-NMI2019} can be understood as \textit{recursive least-squares applied to deep networks in a layer-wise manner}.

If the batch size is large enough, then  $\bX_{1}^{(j)} (\bX_{1}^{(j)})^\top$ is in general invertible, and if $\lambda=0$ then $I- (\bX_{1}^{(j)})^\top \big(  \bX_{1}^{(j)} (\bX_{1}^{(j)})^\top \big)^{-1} \bX_{1}^{(j)}  $ is exactly a projection onto the nullspace of $\bX_{1}^{(j)}$. While \citet{Guo-AAAI2022} showed that updating $\lambda$ in a certain way at every batch improves the empirical performance, inverting $\bX_{1}^{(j)} (\bX_{1}^{(j)})^\top  +\lambda I $ might not be necessary if the goal is to compute (approximate) the nullspace of $\bX_{1}^{(j)}$, which can be done via SVD. In fact, doing SVD would also eliminate the hyper-parameter $\lambda$.

Improvements over the strategy of \citet{Zeng-NMI2019} are given by independent efforts from \citet{Wang-CVPR2021} and \citet{Saha-ICLR2021}. \citet{Wang-CVPR2021} maintain a \textit{weighted}\footnote{In principle (say if the nullspaces of $\bX_i^{(j)}$ intersect), different weights would result in the same nullspace.} sum of covariance matrix $(\bX_{i}^{(j)})^\top \bX_{i}^{(j)}$, where $i$ ranges over previous tasks $1,\dots,t-1$, and then extract an orthonormal basis matrix $K^{(j)}$ the nullspace of some low-rank approximation of this weighted sum. Such $K^{(j)}$ would be the desired $\bK_{t-1}^{(j)}$ defined above, if the intersection $\cap_{i=1}^{t-1}\text{null}(\bX_i^{(j)})$ were not empty. On the other hand, \citet{Saha-ICLR2021} maintain an orthonormal basis for a low-rank approximation of $[(\bX_{1}^{(j)})^\top,\dots, (\bX_{t-1}^{(j)})^\top]$, and they project the gradient onto the orthogonal complement of the subspace spanned by this basis. In light of \eqref{eq:range-null}, the methods of \citet{Saha-ICLR2021} and \citet{Wang-CVPR2021} are in principle equivalent, and their difference in performance (as reported by \citet{Kong-ECCV2022}) relies on implementation details. Furthermore, their approaches are dual to each other and are the \textit{two sides of the same coin}: \citet{Wang-CVPR2021} (approximately) maintains a basis for $\cap_{i=1}^{t-1}\text{null}(\bX_i^{(j)})$ and \citet{Saha-ICLR2021} for the orthogonal complement of $\cap_{i=1}^{t-1}\text{null}(\bX_i^{(j)})$; see Appendix \ref{section:Dual-CPCA} for another example of this duality.

We stress that the provable success of these methods relies on Assumption \ref{assumption:realizability-dagger}, or otherwise pathological cases might arise; recall Example \ref{example:insufficiency} (Past = Present) and see \citet{Lin-ICLR2022} for a different example. Finally, we refer the reader to \citet{Kong-ECCV2022,Liu-ICLR2022,Lin-ICLR2022} for further developments on this line of research.

\subsection{Theoretical Aspects of Continual Learning Through The Lens of Task Relationship}\label{subsection:review-CL-theory}
In order to develop any useful theory for continual learning or related problems that involve multiple tasks, the first step is to make suitable assumptions on the \textit{task relationship} (or \textit{task similarity}, or \textit{task relatedness}). In the literature, there are several candidate assumptions available for modeling the task relationship, which we review here. 

One early assumption on the task relationship is given by \citet{Baxter-JAIR2000} in the context of \textit{learning to learn}, who assumed a statistical data and task generation model such that, a hypothesis space that contains good solutions for a new task can be learned from sufficiently many samples of old tasks. An analogous example can be made in our context: $\ICL^\dagger$ learns $\cK_t$ from old tasks, which serves as the hypothesis space for the next task. 

The task relationship proposed by \citet{Ben-ML2008} involves a family of functions $\cF$. \citet{Ben-ML2008} define two tasks as $\cF$\textit{-related} if the data generation distribution of one task can be transformed into that of the other by some function in $\cF$. This assumption is adopted recently by \citet{Ramesh-ICLR2022} and \citet{Prado-CoLLAs2022} to provide insights on generalization for continual learning. A drawback of this assumption is that the $\cF$-relatedness is typically characterized by an abstract and combinatorial notion of the \textit{generalized VC dimension}, which in our opinion is less intuitive and less ``tangible'' than Assumptions \ref{assumption:realizability-dagger} and \ref{assumption:realizability}. 

\citet{Lee-ICML2021} and \citet{Asanuma-JPSJ2021} contextualize the task relationship in a teacher-student network setup. The tasks are defined by the teacher networks and the similarity of two given tasks is quantified by the difference in the weights of the two corresponding teacher networks. It appears to us that the insights \citet{Lee-ICML2021} and \citet{Asanuma-JPSJ2021} delivered under this teacher-student assumption are now limited to only two tasks.

Yet another task relationship is that every objective $L_t$ is the composition of two functions $f_t\circ g$, where the function $g:\bbR^n\to \bbR^s$ is referred to as the \textit{shared representation}; see, e.g., \citet{Maurer-JMLR2016,Balcan-CoLT2015,Tripuraneni-ICML2021,Du-ICLR2021}. The point is that, if $g$ has been learned from old tasks and fixed, and if $n\gg s$, learning a new task would be more sample-efficient as the dimension of the search space reduces from $n$ to a much smaller $s$. \citet{Li-arXiv2022v2} and \citet{Cao-AISTATS2022} recently explored this idea in the continual learning context with certain new insights, while their settings are from ours. It is possible to obtain stronger generalization bounds and improve our results by leveraging this assumption of a shared representation, at the cost of obtaining a less general theorem than Theorems \ref{theorem:CL1} and \ref{theorem:CL2}.

The task relationship that we consider in this paper is defined by Assumption \ref{assumption:realizability-dagger} or \ref{assumption:realizability}, that is all tasks share at least a common global minimizer. As mentioned in the main paper, this assumption generalizes those of \citet{Evron-COLT2022} and \citet{Pengb-NeurIPS2022}. While we show that Assumptions \ref{assumption:realizability-dagger} and \ref{assumption:realizability} are quite powerful and leads to many pleasant results, in the future we would like to study how strong Assumption \ref{assumption:realizability-dagger} or \ref{assumption:realizability} is for continual learning and come up with other alternatives.

\section{Set-Theoretical Estimation}\label{section:set-theoretical-est}
% \setlength{\epigraphwidth}{.63\textwidth}
% \vspace{-2mm}
% \epigraph{\textit{There is an exciting world ahead of deep-learning networks that can learn continually.}}{\citet{Sutton-talk2022} }
% \vspace{-3mm}
In the main paper, we mentioned that $\ICL^\dagger$ is related to \textit{set-theoretical estimation} in control. Here, we elaborate on this statement by reviewing related works and interpreting $\ICL^\dagger$ from a control and set-theoretical perspective. 

\subsection{Ideal Continual Learner$^\dagger$ $=$ Predictor + Corrector}
\textit{State estimation} is a fundamental problem in control. From the data $D_0,D_1,\dots,D_T$ observed sequentially over time and the initial state $\bw_0\in \cW\subset \bbR^n$, the goal of state estimation is to estimate the subsequent states $\bw_1,\dots,\bw_T\in \bbR^n$, assuming the following equations\footnote{This is a simplified model, e.g., we did not describe the \textit{control signals} that enable state transfer, but this is enough for illustration.}:
\begin{equation}\label{eq:state-estimation}
	\begin{split}
		\bw_{t} &= f_{t-1}(\bw_{t-1}) \\ 
		D_{t-1} &= h_{t-1}(\bw_{t-1})
	\end{split}  \ \ \ \ \ \ \  t=1,\dots, T
\end{equation}
Here, $f_t$ is some known function that transfer state $\bw_t$ into $\bw_{t+1}$, and $h_t$ is a known function that maps state $\bw_{t}$ into our observation $D_t$. Of course, $T$ is allowed to be $\infty$. 

Since the initial state $\bw_0$ belongs to $\cW$ and the subsequent state $\bw_t$ fulfills \eqref{eq:state-estimation}, $\bw_t$ must belong to some set depending on $\cW$, data $D_0,\dots,D_T$, and \eqref{eq:state-estimation}. Such dependency was described by \citet{Witsenhausen-1968} in the case where $f_{t-1}$ and $h_{t-1}$ are linear maps, and was made mathematically precise by \citet{Kieffer-CDC1998} with a recursive algorithm. Defining $K_0\gets\cW$, the algorithm of \citet{Kieffer-CDC1998} consists of two steps for every $t=1,\dots,T$:
\begin{equation*}
	\begin{split}
		\textnormal{(Prediction)}& \ \ \ \ \  G_t \gets f_{t-1}(K_{t-1}) \\ 
		\textnormal{(Correction)} &\ \ \ \ \ K_t \gets  h_{t}^{-1}(D_{t}) \cap G_t 
	\end{split} 
\end{equation*}
In words, the prediction step collects \textit{all possible} states  $G_t$ that can be transferred from $K_{t-1}$, and the correction step rules out the states that are inconsistent with data $D_t$. 

In light of \citet{Kieffer-CDC1998} and as the notations suggest, we have the following description of $\ICL^\dagger$:
\begin{prop}[Ideal Continual Learner$^\dagger$ = Predictor + Corrector]\label{prop:control}
	Under Assumption \ref{assumption:realizability-dagger}, the recursion \eqref{eq:task-t-continual-dagger} of $\ICL^\dagger$ is equivalent to
	\begin{equation}\label{eq:CL-dagger-PC}
		\begin{split}
			\textnormal{(Prediction)}& \ \ \ \ \  \cG_t \gets \argmin_{w\in \cW} L_t(\bw; D_t) \\ 
			\textnormal{(Correction)} &\ \ \ \ \ \cK_t \gets  \cK_{t-1} \cap \cG_t 
		\end{split} 
	\end{equation}
\end{prop}
The prediction step \eqref{eq:CL-dagger-PC} involves solving a single task, while the correction step requires computing the intersection of two complicated sets. Finally, the prediction-correction description also requires storing the knowledge representation $\cK_t$.

\subsection{Ideal Continual Learner$^\dagger$ $=$ Set Intersection Learner$^\dagger$?}
\myparagraph{Background} \citet{Combettes-IEEE1993} formulated a general estimation problem that encompasses many signal processing and control applications. The formulation of \citet{Combettes-IEEE1993} can be described as follows. Suppose we are given some information $D_1,\dots,D_T$, from which we want to estimate a certain variable $\hat{\bw}$. Each piece of information $D_t$ determines a set $\cG_t$ in which $\bw^*$ lies. Assuming we can compute $G_t$ from $D_t$, the problem of estimating $\hat{\bw}$ reduces to computing (a point of) the set intersection $\cap_{t=1}^TG_t$.  If the information we have is \textit{inconsistent} (e.g., due to noise or out-of-distribution measurements), the intersection $\cap_{t=1}^TG_t$ might be empty. This is a more challenging situation, which might result in an NP-hard problem (as we discussed in the main paper). As in Assumption \ref{assumption:realizability-dagger}, we assume $\cap_{t=1}^TG_t$ is non-empty, therefore the difficulty of finding a point in $\cap_{t=1}^TG_t$ largely depends on the \textit{shape} of each set $G_t$.

\myparagraph{Computing Set Intersections} Suppose $\cap_{t=1}^TG_t\neq \varnothing$. Finding a point in $\cap_{t=1}^TG_t$ is a classic problem and has been of great interest nowadays due to many modern applications. Classic methods include the alternating projection method of John von Neumann in 1933 for the case of two intersecting affine subspaces (as \citet{Lindstrom-2021} surveyed); the methods of \citet{Karczmarz-1937} and \citet{Cimmino-1938} for the case where $G_t$'s are  finitely many affine hyperplanes (as \citet{Combettes-IEEE1993} reviewed); the method of \citet{Boyle-1986} for computing a point in the intersection of two convex sets; as well as many extensions of these methods for finitely many convex sets and even non-convex sets. We refer the reader to \citet{Borwein-2014} for some equivalence among these methods and to  \citet{Theodoridis-SPM2010,Lindstrom-2021} and some references of \citet{Borwein-2014} for surveys. See also \citet{Evron-COLT2022} where some extensions of \citet{Karczmarz-1937} were discussed in a continual learning context.

\myparagraph{The Continual Learning Context} The Ideal Continual Learner$^\dagger$ ($\ICL^\dagger$) might be viewed as a method that computes the set $\cG_t$ of global minimizes for each task $t$ and then computes their intersection $\cap_{t=1}^T \cG_t$. If all objectives $L_t$'s are convex, then in general $\cG_t$ is a convex set, and it is possible to implement $\ICL^\dagger$ via computing intersections of convex sets. However, there are still several issues that make the continual learning problem harder than computing set intersections:
\begin{itemize}
    \item The first difficulty is to find a representation of $\cG_t$ that is amenable to alternating optimization  (we discussed several equivalent representations of $\cG_t$ in \S \ref{section:optimization}). Such representation should consume as less memory as possible, and yet might purely rely on a user's choice; e.g., one might revisit past tasks as an implicit way of storing $\cG_t$ \citep{Evron-COLT2022}.
    \item The other difficulty occurs when there are infinitely many tasks ($T=\infty$), and this basically rules out the chance of storing all sets $\cG_t$'s. Similar issues have been considered in the signal processing context, where one is given the set $\cG_t$ sequentially; see \citet{Theodoridis-SPM2010} for a review. However, theoretical guarantees of existing methods in that line of research are about convergence to $\cap_{t=t_0}^\infty \cG_t$ for some number $t_0$  (potentially unknown), not to $\cap_{t=1}^\infty \cG_t$. The reason that $t_0$ is in general unknown is this: These methods are some variants of alternating projections, which might \textit{forget} the very first set $\cG_1$ after projecting onto later sets.
\end{itemize}

\section{Streaming PCA, Incremental SVD, and Subspace Tracking}\label{section:streaming-PCA}
As mentioned in the main paper, the continual matrix factorization formulation is related to several classic problems, \textit{streaming principal component analysis} (PCA), \textit{incremental singular value decomposition} (SVD), and \textit{subspace tracking}. These problems have been extensively studied in machine learning, numerical linear algebra, and signal processing communities. While we refer the reader to \citet{Balzano-IEEE2018} and \citet{Vaswani-SPM2018} for comprehensive reviews, we discuss several related methods here, aiming at highlighting their connections to continual learning.

\myparagraph{Streaming PCA} PCA refers to modeling a data matrix $[\bY_1\ \cdots\ \bY_T]$ with some low-dimensional subspace, and it is a century-old topic \citep{Pearson-1901} that has many applications and extensions \citep{jolliffe-PCA2002,Vidal-GPCA2016}. Streaming PCA differs from PCA in that one receives $\bY_1,\dots,\bY_T$ sequentially, and each $\bY_t$ is seen only once. A popular method for streaming PCA is due to \citet{Oja-1982}. Assume we want to project the data onto a $r$-dimensional subspace of $\bbR^n$, with $r$ a hyper-parameter. The method of \citet{Oja-1982} initializes a random orthonormal matrix $\bU^{(1)}\in\bbR^{n\times r}$, updates
\begin{align}\label{eq:Oja}
    \bU^{(k+1)} \gets \textnormal{orth} \Big( \bU^{(k)}  + \gamma^{(k)} \bY_k \bY_k^\top  \bU^{(k)}  \Big)
\end{align}
at iteration $k$ with stepsize $\gamma^{(k)}$. Here  $\textnormal{orth}(\cdot)$ orthogonalizes an input matrix (e.g., via QR or SVD decomposition). The update \eqref{eq:Oja} can be regarded as a stochastic projected gradient descent method; in particular, $\gamma^{(k)}\to \infty$, then \eqref{eq:Oja} becomes a stochastic version of the power method that computes $r$ eigenvectors of $[\bY_1\ \cdots\ \bY_T]$ corresponding to $r$ maximum eigenvalues. Similarly to SGD or the algorithm of \citet{Evron-COLT2022}, in the worse case, this stochastic method forgets if it does not revisit past data, and even if  $[\bY_1\ \cdots\ \bY_T]$ is exactly of rank $r$; note that we did not give a formal proof for this claim, but it can be derived by combining the minimality of $\cK_t$ (\S \ref{subsection:2principles}) with problem-specific arguments.

\myparagraph{Incremental SVD} The incremental SVD problem is to compute the SVD of $[\bY_1\ \bY_2]$ from the SVD of a matrix $\bY_1=\bU_1 \bm{\Sigma}_1 \bV^\top_1$. This is a different and slightly more difficult goal from continual matrix factorization, which aims to estimate (an orthonormal basis matrix of) a subspace sum. Of course, we could do incremental SVD for increasingly more data, and extract the desired basis matrix from the final SVD decomposition of $[\bY_1\ \cdots\ \bY_t]$; in other words, incremental SVD algorithms can be used to solve the continual matrix factorization problem. The issue is that it takes more memory to store the SVD than storing only a basis matrix as in our approach, and potentially takes more time for computation. 

To better understand the matter, we review a  popular incremental SVD method, described in Section 3 of \citet{Brand-ECCV2002} (see also \citet{Bunch-1978}). 
With $Y_2:=(I - \bU_1 \bU_1^\top) \bY_2$ and a QR-decomposition $Q_2R_2=Y_2$ of $Y_2$, we have
\begin{align*}
    [\bY_1 \ \ \bY_2] &= [\bU_1 \bm{\Sigma}_1 \bV^\top_1\  \bY_2] \\
    &= [\bU_1\ \  Q_2 ] \begin{bmatrix}
        \bm{\Sigma}_1 & \bU_1^\top \bY_2 \\
        0&R_2
    \end{bmatrix} \begin{bmatrix}
        \bV_1 & 0 \\
        0 & I
    \end{bmatrix}^\top
\end{align*}
Since $[\bU_1\ \  Q_2 ]$ and $\begin{bmatrix}
        \bV_1 & 0 \\
        0 & I
    \end{bmatrix}$ have orthonormal columns, computing the SVD of $[\bY_1 \ \ \bY_2]$ reduces to computing the SVD $U_2 \bm{\Sigma}_2 V_2^\top$ of the middle matrix $\begin{bmatrix}
        \bm{\Sigma}_1 & \bU_1^\top \bY_2 \\
        0&R_2
    \end{bmatrix}$. Indeed, with SVD $U_2 \bm{\Sigma}_2 V_2^\top$, we can obtain the desired decomposition of $[\bY_1 \ \ \bY_2]$ as $\bU_2 \bm{\Sigma}_2 \bV_2^\top$, where $\bU_2$ and $\bV_2$ are defined as
\begin{align}\label{eq:post-multiply}
    \bU_2\gets [\bU_1\ \  Q_2 ] U_2, \ \ \bV_2 = \bV_1 V_2.
\end{align}

Note that this approach can also be regarded as an \textit{\textbf{expansion-based method}}, and by design it never forgets when applied to continual matrix factorization. However, it requires an extra QR decomposition, and it requires an SVD on a larger matrix than in our approach (\S \ref{subsection:CMF}). A  subtler drawback is that this method performs SVD and then multiplication \eqref{eq:post-multiply} to obtain the orthonormal basis. This is potentially numerically unstable; as \citet{Brand-LAA2006} later commented: 
\begin{center}
    \textit{Over thousands or millions of updates, the multiplications may erode the orthogonality of $\bU'$ [e.g., $\bU_2$] through numerical error... [\citet{Brand-LAA2006} described some ad-hoc remedy] ... It is an open question how often this is necessary to guarantee a certain overall level of numerical precision. }
\end{center}
This numerical issue does not exist in our approach, where our multiplication is followed by an SVD; then from the SVD we directly extract the desired basis matrix, which is typically orthonormal up to machine accuracy given the numerical stability of SVD algorithms (preconditioned if necessary). The reason that we could do so, though, stems from the fact that we need only to compute an orthonormal basis of the subspace sum $\cS_1+\cdots + \cS_t$ for the continual matrix factorization problem, which is arguably simpler than incrementally updating the SVD. As a final remark, our approach can be extended to update the $U$ and $V$ matrices for incremental SVD, at the potential cost that updating the singular values could be numerically unstable. Nevertheless, discussing these is beyond the scope of the paper, and hence we omit the details.

\myparagraph{Subspace Tracking} Suppose now we receive data $\bY_1,\dots,\bY_T$ sequentially, with $\cS_t=\text{range}(\bY_t)$. In the problem of subspace tracking, we assume $\cS_t$ changes slowly, as $t$ increases, and the goal at task $t$ or time $t$ is to estimate $\cS_t$. To do so, a classic idea (e.g., see \citet{Yang-TSP1995,Yang-SP1996}) is to define a \textit{forgetting} factor $\beta$ with $0\ll \beta \leq 1$, and minimize  the following objective in a streaming fashion:
\begin{align}
    \min_{\bU} \sum_{t=1}^T \beta^{T-t} \| \bY_t - \bU\bU^\top \bY_t\|_{\textnormal{F} }^2.
\end{align}
Here, the forgetting factor $\beta$ discounts the previous samples, accounting for the fact that $\cS_t$ is slowly changing; more generally, such factor $\beta$ also appears in many classic reinforcement learning algorithms as a \textit{discount rate} to discount the past rewards \citep{Sutton-book2018}. There are also works that make mathematically precise assumptions on \textit{how} $\cS_t$ changes over time. For example, $\cS_t$ could be slowly rotated, or keep static for a sufficiently long time until changing \citep{Narayanamurthy-TIT2018}. Intuitively, making how $\cS_t$ changes precise give chances to design better algorithms, although we shall not review these approaches. The point here is that, this classic line of research furnished the very rudimentary idea of what is recently known as \textit{selective forgetting}  or \textit{active forgetting}  in the deep continual learning context \citep{Aljundi-ECCV2018,Wang-NeurIPS2021,Shibata-IJCAI2021,Liu-CoLLAs2022}. Of course, achieving selective or active forgetting in a principled way for deep continual learning is  much harder than for subspace tracking.

\section{A Dual Approach to Continual Principal Component Analysis}\label{section:Dual-CPCA}
\myparagraph{Background} In \S \ref{subsection:CMF}, we considered continual matrix factorization where task $t$ is associated with the loss $L_t\big((\bU,\bC); \bY_t \big) = \| \bU \bC -\bY_t  \|_\textnormal{F}^2$, and this loss is equivalent to a PCA objective  $ \| \bU \bU^\top \bY_t -\bY_t  \|_\textnormal{F}^2$ in variable $\bU$, assuming $\bU^\top\bU$ is the identity matrix. Since $\bU$ gradually grows its columns to accommodate new tasks, this might eventually consume too much memory. As promised in Remark \ref{remark:grow},  we now address this issue, using the following key observation: If $\bU$ and $\bB$ are orthonormal basis matrices\footref{footnote:matsize} of some linear subspace and its orthogonal complement, respectively, then $\bU\bU^\top +\bB\bB^\top$ is the identity matrix, and we have
\begin{align}\label{eq:Primal=Dual}
    \| \bU \bU^\top \bY_t -\bY_t  \|_\textnormal{F}^2 = \| \bB \bB^\top\bY_t  \|_\textnormal{F}^2 = \| \bY_t^\top \bB \|_\textnormal{F}^2.
\end{align}
Note that relation \eqref{eq:Primal=Dual} was discussed by \citet{Tsakiris-JMLR2018} in a different context, in comparison to the algorithm of \citet{Lerman-FoCM2015}. In that line of research, $\bU$ is typically referred to as the \textit{primal representation} of the subspace $\cS_t$ and $\bB$ the dual representation. To conclude, we can use $\| \bY_t^\top \bB \|_\textnormal{F}^2$ as the objective for task $t$, and the aim for task $t$ is to find an orthonormal basis matrix for subspace $\cS^\perp_t$, which is uniquely identified with $\cS_t$. The goal of \textit{continual principal component analysis} is to find the subspace intersection $\cap_{i=1}^t\cS^\perp_i$ upon encountering task $t$ ($\forall t$). 

\myparagraph{Implementing The Ideal Continual Learner$^\dagger$} We now describe $\ICL^\dagger$ for solving the problem of continual PCA. Denote by $B_t$ a basis matrix for $\cap_{i=1}^t\cS^\perp_i$ with $B_t^\top B_t=I$, we have $\| \bY_i^\top B_t \|_\textnormal{F}=0$ for every $i=1,\dots, t$. Therefore, for each task, we maintain and update $B_t$. To start with, set $B_1$ to be the matrix whose columns are all right singular vectors of $\bY_t^\top$ corresponding to its zero singular values. Suppose now $t>1$ and we are given $B_{t-1}$ and $\bY_t$, and we wish to compute a $B_t$. Of course $\cap_{i=1}^t\cS^\perp_i \subset \cap_{i=1}^{t-1}\cS^\perp_i$, hence $B_t$ has fewer or equal columns than $B_{t-1}$ and each column of $B_t$ can be represented as a linear combination of columns of $B_{t-1}$. Therefore, parameterize $B_t$ as $B_t=B_{t-1} \overline{C}_t$, and we solve
\begin{align*}
    &\ \overline{C}_t\in\argmin_{C_t} \| \bY_t^\top B_{t-1} C_t\|_\textnormal{F}^2 \ \ \text{s.t.}  \ \ (B_{t-1} C_t)^\top B_{t-1} C_t = I \\ 
    \Leftrightarrow & \ \overline{C}_t\in\argmin_{C_t} \| \bY_t^\top B_{t-1} C_t \|_\textnormal{F}^2 \ \ \text{s.t.}  \ \  C_t^\top C_t = I,
\end{align*}
where the equivalence is due to $B_{t-1}^\top B_{t-1} = I$. As a result, we can set $\overline{C}_t$ to be the matrix whose columns are all right singular vectors of $\bY_t^\top B_{t-1}$ corresponding to its zero singular values. This concludes how $\ICL^\dagger$ can be implemented for continual PCA.

\myparagraph{Summary} We now see \textit{two sides of the same coin}: We can implement $\ICL^\dagger$ for continual matrix factorization (or continual PCA), using either the primal or dual representation. The primal representation ($U_t$) grows while the dual representation ($B_t$) shrinks. This is not to say that the dual representation is always better. The dual representation saves memory (and is thus preferred) if and only if $\dim(\sum_{i=1}^t \cS_i)$ exceeds half of the ambient dimension. Hence, a better implementation for continual matrix factorization  (or continual PCA) would switch the representation from primal to dual, as $t$ increases such that $\dim(\sum_{i=1}^t \cS_i)$ is greater than that threshold. Nevertheless, the total memory needed in both representations is bounded above by $n\times \dim(\sum_{i=1}^t \cS_i)$. This seems to contradict the results of \citet{Knoblauch-ICML2020} and \citet{Chen-FOCS2022}, which suggest that continual learning without forgetting would require memory to grow linearly with the number $T$ of tasks. The catch is that, besides other differences in settings, their results are for the worst-case scenario, while our continual matrix factorization or continual PCA setting is quite specific and does not belong to the worst case.

\section{Proofs}
\begin{proof}[Proof of Proposition \ref{prop:sufficiency}]
	Clearly $\cK_1 = \cG_1$. Assume $\cK_{t-1}=\cap_{i=1}^{t-1} \cG_i$. Since $L_t$ is minimized over $\cG_t$ \eqref{eq:task-t-dagger} and $\cap_{i=1}^t \cG_i\neq \varnothing$, the set of global minimizers of \eqref{eq:task-t-continual-dagger} is  $\cap_{i=1}^t \cG_i$. By definition \eqref{eq:task-t-continual-dagger} we have $\cK_t =\cap_{i=1}^t \cG_i$. The proof is complete by induction.
\end{proof}

\begin{proof}[Proof of Proposition \ref{prop:multitask}]
	Let $\hat{\bw}_t\in \cK_t=\cap_{i=1}^t\cG_i$,  then we have %i.e., $\hat{\bw}_t$ is a common global minimizer of $L_1,\dots,L_t$ over $\cW$. Then we have 
	\begin{equation}
		\begin{split}
		%	\sum_{i=1}^t  \alpha_i L_i(\hat{\bw}_t; D_i)  &\geq \min_{w\in \cW} \sum_{i=1}^t  \alpha_i L_i(\bw; D_i)  \\
			\min_{w\in \cW} \sum_{i=1}^t  \alpha_i L_i(\bw; D_i) &\geq \sum_{i=1}^t \alpha_i \min_{w\in \cW}    L_i(\bw; D_i) \\
			&=\sum_{i=1}^t \alpha_i  L_i(\hat{\bw}_t; D_i).
		\end{split}
	\end{equation}
	Since $\hat{\bw}_t$ is feasible to the multitask objective \eqref{eq:CL=multitask}, the minimum value of the multitask objective is attained exactly when each objective $L_i$ is minimized, meaning that $\cK_t$ is the set of global minimizers of the multitask objective.
\end{proof}

\begin{proof}[Proof of Theorem \ref{theorem:CL1}]
    Assumption \ref{assumption:realizability-dagger} and Proposition \ref{prop:sufficiency} imply $\hat{\bw}$ is a common global minimizer of all empirical tasks \eqref{eq:task-t-dagger}. For $t=1,\dots,T$, let $\bw^*_t\in \cG^*_t$ be a global minimizer of learning task $t$ \eqref{eq:population}. Applying a union bound to \eqref{eq:uniform-convergence}, we get
	\begin{equation}\label{eq:some-inequality}
		\begin{split}
			\bbE_{\bd\sim \cD_t }[ \ell_t( \hat{\bw}; \bd ) ] - L_t(\hat{\bw}; D_t) &\leq \zeta(m_t,\delta'), \\
			L_t(\bw^*_t; D_t) - \bbE_{\bd\sim \cD_t }[ \ell_t( \bw^*_t; \bd ) ]  &\leq \zeta(m_t,\delta'),
		\end{split}
	\end{equation}
    for all $t=1,\dots,T$, with probability at least $1-T\delta'$. Summing up the inequalities of \eqref{eq:some-inequality} and noticing $L_t(\hat{\bw}; D_t)\leq L_t(\bw^*_t; D_t) $ and $c^*_t = \bbE_{\bd\sim \cD_t }[ \ell_t( \bw^*_t; \bd ) ]$, we then obtain
	\begin{align*}
		\bbE_{\bd\sim \cD_t }[ \ell_t( \hat{\bw}; \bd ) ]  \leq c^*_t + 2\zeta(m_t,\delta') = c^*_t + \zeta(m_t,\delta').
	\end{align*}
    The last equality is due to the big-$O$ notation \eqref{eq:uniform-convergence}. With $\delta:=\delta'/T$ we finish the proof.
\end{proof}

\begin{proof}[Proof of Proposition \ref{prop:implement-ICL-CLR}]
    Given $(\bX_1,\by_1)$, we can compute $\bK_1$ via an SVD on $\bX_1$, e.g., set $\bK_1$ to be the matrix whose columns are right singular vectors of $\bX_1$ corresponding to its zero singular values. We can compute a particular solution $\hat{\bw}_1:= \bX_1^{\dagger} \by_1$ to the normal equations $\bX_1^\top \bX_1 \bw = \bX_1^\top \by_1$; here $\bX_1^{\dagger}$ denotes the pseudoinverse of $\bX_1$, and can be computed from SVD of $\bX_1$. 
    
    For the case $t>1$, suppose we are given $(\hat{\bw}_{t-1},\bK_{t-1})$ and data $(\bX_t,\by_t)$. Then, by the definitions of $(\hat{\bw}_{t-1},\bK_{t-1})$ and $\cK_{t-1}$, the equivalence between \eqref{eq:continual-CLR} and \eqref{eq:continual-CLR-unconstrained} is immediate. 

    The normal equations of \eqref{eq:continual-CLR-unconstrained} are given by ($\overline{\bX}_t:= \bX_t \bK_{t-1}$)
    \begin{align*}
    	\overline{\bX}_t^\top \overline{\bX}_t \ba = \overline{\bX}_t^\top (\by_t -  \bX_t \hat{\bw}_{t-1}).
    \end{align*}
    Therefore, we can compute $\hat{\ba}_t= \overline{\bX}_t^\dagger (\by_t -  \bX_t \hat{\bw}_{t-1})$ as  a global minimizer of \eqref{eq:continual-CLR-unconstrained}, where the pseudoinverse $\overline{\bX}_t^\dagger$ can be calculated via an SVD on $\overline{\bX}_t$. Then a common global minimizer $\hat{\bw}_t\in\cK_t$ is given as $\hat{\bw}_t=\hat{\bw}_{t-1} + \bK_{t-1} \hat{\ba}_t$. It remains to compute an orthonormal basis matrix $\bK_t$ for the intersection of the nullspaces of $\bX_1,\dots,\bX_t$, that is the intersection of the nullspace of $\bX_t$ and the range space of $\bK_{t-1}$. This can be done as follows. First, compute an orthonormal basis for the nullspace of $\overline{\bX}_t$, then left-multiply this basis by $\bK_{t-1}$. This yields the desired $\bK_t$.
\end{proof}

\begin{proof}[Proof of Proposition \ref{prop:approximately-optimal}]
	The proof is obtained in the same way as in Theorem \ref{theorem:CL1}, by exchanging the empirical part (e.g., $L_t(\bw,D_t)$ and $c_t$) with the population part (e.g., $\bbE_{\bd\sim \cD_t }[ \ell_t(\bw; \bd ) ]$ and $c^*_t$).
\end{proof}

\begin{proof}[Proof of Theorem \ref{theorem:CL2}]
	The feasibility of the relaxed Continual Learner$^\dagger$ follows from Proposition \ref{prop:approximately-optimal}, which also suggests that
	\begin{align}\label{eq:optimal=>feasible}
		L_t(\overline{\bw}; D_t) \leq c_t + \zeta(m_t,\delta/T), \ \ \forall t=1,\dots,T
	\end{align}
	with probability at least $1-\delta$. Note that in \eqref{eq:optimal=>feasible} the uniform convergence bound of Assumption \ref{assumption:uniform-convergence} was invoked for each task, meaning that we have $\bbE_{\bd\sim \cD_t }[ \ell_t( \overline{\bw}; \bd ) ] \leq L_t(\overline{\bw}; D_t) + \zeta(m_t,\delta/T)$ and $c_t \leq c^*_t + \zeta(m_t,\delta/T)$. Substitute them into \eqref{eq:optimal=>feasible} to get $\bbE_{\bd\sim \cD_t }[ \ell_t( \overline{\bw}; \bd ) ]  \leq c^*_t + 3\zeta(m_t,\delta/T)$, finishing the proof.
\end{proof}

\begin{proof}[Proof of Theorem \ref{theorem:rehearsal}]
    By Proposition \ref{prop:approximately-optimal}, with probability at least $1-\delta$, we have $\overline{\bw}$ satisfying
    \begin{align*}
        \frac{1}{m_T}\sum_{i=1}^{m_T}\ell_T (\overline{\bw}; \bd_{Ti}) + \sum_{t=1}^{T-1}\frac{1}{s_t}\sum_{i=1}^{s_t}\ell_t (\overline{\bw}; \bd_{ti}) \leq  c_T + \zeta(m_T, \delta/T) +  \sum_{t=1}^{T-1} \big( \hat{c}_t + \zeta(s_t,\delta/T) \big),
    \end{align*}
    where $\hat{c}_t$ is the minimum of $\frac{1}{s_t}\sum_{i=1}^{s_t}\ell_t (\bw; \bd_{ti})$ over $\cW$, and $c_T$ the minimum of $\frac{1}{m_T}\sum_{i=1}^{m_T}\ell_T (\bw; \bd_{Ti})$ over $\cW$. Then, the proof technique of Theorem \ref{theorem:CL2} implies \eqref{eq:generalization-bound-CL3}, and the proof is complete.
\end{proof}

\end{document}